\documentclass[11pt]{article}
\usepackage[margin=1.09in]{geometry}
\usepackage[utf8]{inputenc}

\usepackage{amsthm}
\newtheorem{theorem}{Theorem}[section]
\newtheorem{corollary}[theorem]{Corollary}
\newtheorem{lemma}[theorem]{Lemma}
\newtheorem{remark}[theorem]{Remark}

\newtheorem{proposition}[theorem]{Proposition}

\usepackage[utf8]{inputenc}
\usepackage{amsmath}
\usepackage{mathtools}
\usepackage{amssymb}
\usepackage{amsopn}
\usepackage{bbm}
\usepackage{color}
\usepackage{hyperref}
\usepackage{url}
\usepackage{tikz}
\usepackage{tabu}
\usepackage[font=footnotesize,labelfont=bf]{caption}
\usepackage{enumerate}
\usepackage{placeins}
\usepackage{eso-pic}
\usepackage{caption}

\hypersetup{
  colorlinks = true, 
  urlcolor = blue, 
  linkcolor = blue, 
  citecolor = red 
}
\usepackage{mathtools}
\PassOptionsToPackage{mathcal}{euscript}

\newcommand{\diag}{\mathrm{diag}}
\newcommand{\Var}{\mathbb{V}\mathrm{ar}}
\newcommand{\deltal}{\delta_L}
\newcommand{\psd}{\mathrm{psd}}

\newcommand{\abs}[1]{\left\vert#1\right\vert}
\newcommand{\norm}[1]{\left\lVert#1\right\rVert}
\newcommand{\tensornorm}[1]{{\left\vert\kern-0.25ex\left\vert\kern-0.25ex\left\vert #1 
    \right\vert\kern-0.25ex\right\vert\kern-0.25ex\right\vert}}
\newcommand{\floor}[1]{\left\lfloor#1\right\rfloor}
\newcommand{\overbar}[1]{\mkern 1mu \overline{\mkern-1.5mu#1\mkern0mu}\mkern 1mu}
\newcommand{\const}{\,\mathrm{const}}
\newcommand{\decay}{\,\mathrm{decay}}

\newcommand{\E}{\mathbb{E}}
\newcommand{\N}{\mathbb{N}}
\newcommand{\Prob}{\mathbb{P}}
\newcommand{\R}{\mathbb{R}}
\newcommand{\Ss}{\mathbb{S}}

\newcommand{\EE}{\mathcal{E}}

\newcommand{\OO}{\mathcal{O}}

\newcommand{\UU}{\mathcal{U}}

\newcommand{\dd}{\mathrm{d}}
\newcommand{\del}{\partial}

\newcommand{\tr}{\mathrm{tr}}
\newcommand{\one}{\mathbbm{1}}

\newcommand{\lr}{\eta}
\newcommand{\de}{\delta^{(L)}}

\newcommand{\vecc}{\mathrm{vec}}

\newcommand{\Ca}{c_{\alpha}}

\newcommand{\qedwhite}{\begin{flushright} \ensuremath{\Box} \end{flushright}}

\newcommand{\parms}{\alpha}

\newtheorem{assumption}[theorem]{Assumption}

\title{Convergence and   Regularization Properties\\ of Gradient Descent for Deep Residual Networks\footnote{Alain Rossier's research was supported through EPSRC Centre for Doctoral Training in Mathematics of Random Systems: Analysis, Modelling and Simulation (EP/S023925/1). 
}}
\author{Rama~Cont$^1$, Alain Rossier$^{1,2}$ and Renyuan Xu$^3$\\
\ \\
   \small{ $^1$ Mathematical Institute, University of Oxford \hspace{0.2cm} $^2$ Instadeep Ltd}\\
   \small{ $^3$ Department of Industrial and Systems Engineering, University of Southern California} \\
   \small{rossier@maths.ox.ac.uk, Rama.Cont@maths.ox.ac.uk, renyuanx@usc.edu}
}

\begin{document}
\numberwithin{equation}{section}

\maketitle

\begin{abstract}
We prove linear convergence of gradient descent to a global optimum for the training of deep residual networks with constant layer width and smooth activation function. We show that if the trained weights, as a function of the layer index, admit a scaling limit as the depth increases, then the limit has finite $p-$variation with $p=2$. Proofs are based on non-asymptotic estimates for the loss function and for norms of the network weights along the gradient descent path. We illustrate the relevance of our theoretical results to practical settings using detailed numerical experiments on supervised learning problems.
\end{abstract}
\tableofcontents
\newpage
\section{Introduction}

Whether gradient descent methods find globally optimal solutions in the training of neural networks and how trained neural networks generalize are two major open questions in the theory of deep learning. The non-convexity of the loss functions for neural network training may lead to sub-optimal solutions when applying gradient descent methods. It is thus relevant to understand from a theoretical point of view whether specific neural network architectures with a proper choice of learning rates for gradient descent methods can improve the optimization landscape and/or eliminate sub-optimal solutions \cite{sun2020global}. There is some empirical evidence that gradient descent seems to select solutions that generalize well  \cite{zhang2021understanding} even without any {\it explicit} regularization. Hence, it is believed that gradient descent induces an implicit regularization \cite{neyshabur2014search} and characterizing the nature of this regularization is an interesting research question. 

In the present work we prove linear convergence of gradient descent to a global minimum for a class of deep residual networks with constant layer width and smooth activation function. Furthermore, we show that under practical assumptions, the trained weights admit a scaling limit as a function of the layer index which has finite $2$-variation. Our result shows that how implicit regularization emerges from gradient descent.
Our proofs are based on non-asymptotic estimates for the loss function and norms of the network weights along the gradient descent path. These non-asymptotic estimates are interesting in their own right and may prove useful to other researchers for the study of dynamics of learning algorithms.

\subsection{Convergence and   regularization properties of deep learning algorithms}

Existing results on convergence and implicit regularization in deep learning exploit 
three paradigms: over-parametrized neural networks with fixed depth and large width, linear neural networks with sufficiently large depth, and mean-field residual networks.

Under sufficient over-parametrization by width with {\it fixed depth}, many popular neural network architectures (including feed-forward, convolutional, and residual) with ReLU activation find a global optimum in linear time with respect to the remaining error and the trained network generalizes well \cite{ALL2019, ALS2019}. However, the associated generalization bounds are intractable, and the amount of over-parametrization implied in these results is often unrealistically large.  One can improve the asymptotic analysis \cite{ZYYL2020, ZG2019}, but it still falls short of leading to any practical insight. For smooth activation functions, \cite{du2019gradient} studied the convergence of gradient descent for various network architectures, including residual networks. They show that for any depth, if the residual layers are wide enough and the learning rate is small enough, gradient descent on the empirical mean-squared loss converges to a solution with zero training loss in linear time. The rate of convergence is proportional to the learning rate and the minimum eigenvalue of the \textit{Gram matrix}. \cite{EMWW2019} showed that in the over-parametrized regime, for a suitable initialization with the last layer initialized at zero and other weights initialized uniformly, gradient descent can find a global minimum exponentially fast with high probability. 

For {\it linear} deep neural networks (i.e. with identity activation function), \cite{bartlett2018gradient} showed that training with gradient descent is able to learn the positive definite linear transformations using identity initialization. \cite{wu2019global} proposed a new initialization scheme named zero-asymmetric (ZAS) and proved that that under such initialization, for an arbitrary target matrix, gradient descent converges to an $\epsilon$-optimal point in $\mathcal{O}(L^3\log(1/\epsilon))$ iterations, which scales polynomially with the network depth $L$. Subsequent refinements of the convergence rates and the width requirements have been established in \cite{DH2019, zou2020global}. Finally, \cite{yun2020unifying} showed the implicit regularization of gradient descent for linear fully-connected networks to $l_2$ max-margin solutions. \\


Another line of work deals with mean-field residual networks by looking at the continuum limit of residual networks when either the depth $L$ or the width $d$ goes to infinity. \cite{YS2017} build on the analysis of \cite{DFS2016} for feed-forward networks to study the average behaviour  of randomly initialized residual networks with  width tending to infinity. They show that a careful initialization, depending on the depth, may enhance expressivity. 
Further, \cite{lu2020mean} proposed a continuum limit of deep residual networks by letting the depth $L$ tends to infinity and showed that every local minimum of the loss landscape is global. This characterization enables them to derive the first global convergence result for multi-layer neural networks in the mean-field regime.


In addition to the network architectures listed above, non-linear neural networks with fixed width and {\it large but finite} depth are successful and practically more popular \cite{HZRS2015, HZRS2016}. It is well-documented that for a fixed number of parameters, going deeper allows the models to capture richer structures \cite{eldan2016power, T2015}. However, the theoretical foundations for such networks remain widely open due to their complex training landscape. 

\subsection{Contributions}

We consider a supervised learning problem where we seek to learn an unknown mapping with inputs and outputs in $\mathbb{R}^d$ using a residual network with constant width $d$ and a smooth activation function.
We study the convergence and implicit regularization of gradient descent for the mean-squared error.  
\begin{itemize}
    \item {\bf Linear convergence.} For $\epsilon > 0$, we prove that for a residual network of depth $L=\Omega(1/\epsilon)$, we can choose a learning rate schedule such that gradient descent on the training loss converges to a $\epsilon$-optimal solution in $\Theta(\log(1/\epsilon))$ iterations.
    \item {\bf Scaling limit of trained weights.} The trained weights, as  a function of the layer, may admit a scaling limit as $L\to\infty$. We prove that such a scaling limit is a matrix-valued function with finite $2$-variation.
    \item {\bf Non-asymptotic estimates on loss function and weights along the gradient descent path.} In addition to the convergence results mentioned above, we obtain (non-asymptotic) estimates along the gradient descent path for the loss function and various norms of the weights, with tractable bounds. 
       
    \item {\bf Relevance to practical settings.} We illustrate the relevance of our  theoretical results in practical settings using  detailed numerical experiments with networks of realistic width and depth. \\
\end{itemize}
Our analysis generalizes previous results on {\it linear} neural networks \cite{wu2019global} to a more general nonlinear setting relevant for learning problems. Our non-asymptotic results stand in contrast to the mean-field analysis \cite{lu2020mean} which requires infinite depth. Our tractable bounds improve upon the ones found for networks over-parametrized by width \cite{ALS2019, du2019gradient, EMWW2019, ZYYL2020, ZG2019}, where the trained weights do not leave the lazy training regime \cite{COB2019}: in our setting, the trained weights are not necessarily staying close to their initialization. A key ingredient in the proof is to study the evolution of various norms for the weights under gradient descent iterations. These estimates are provided in Lemmas \ref{lemma-f} and \ref{lemma-g}. 

Our theoretical results suggest that initialization of weights at scale $L^{-1}$ together with a $L^{-1/2}$ scaling of the activation function leads to convergence under a constant learning rate. The overarching principle is to make sure that the gradient stays on the same scale as the weights (here $L^{-1/2}$) during training.
Our analysis also extends, with minimal changes, to the case where linear layers are added at the beginning and the end of the network. \\ 


\paragraph{Notations}
Define $(e_m)_{m'} = \one_{\{m'=m\}}\in \R^d$. For a vector $x\in\R^d$, we denote $\norm{x}_2$ the Euclidean norm of $x$, and for a matrix $M\in\R^{d\times d}$, we denote $\norm{M}_F$ the Frobenius norm of $M$. 
When the context is clear, we omit the superscript $x$ for the quantities that depend on the input $x$. We denote $f = \OO(g)$ if there exists $c>0$ such that $f(z)\leq cg(z)$, where $z=(L, k, t, \eta_L(t), c_0)$. That means, our Big-O notation involves a constant that is independent of the depth $L$, the layer number $k$, the iteration number $t$, the learning rates $\eta_L(t)$, and the universal constant $c_0$ defined in Assumption \ref{3_4_assumptions_1}. Similar definitions stand for $\Omega$ and $\Theta$. For a function $\sigma \colon \R \to \R$,  define $\sigma_d \colon \R^d \to \R^d$ by $\sigma_d(x)_i = \sigma(x_i)$ for $i=1, \ldots, d$.

\section{Residual networks}

Let $x\in\R^d$ be an input vector, $\deltal$ be a fixed positive real number, and $\parms^{(L)} \in \R^{L\times d \times d}$ be a set of parameters (or weights). In this section, we focus on a ResNet architecture \textit{without bias} with $L$ fully-connected layers: 
\begin{equation} \label{resnets_wo_bias}
    \begin{cases}
        h^{x,\, (L)}_{k} &= h^{x,\, (L)}_{k-1} + \deltal \sigma_d\hspace{-1pt} \left(\parms_k^{(L)} h^{x, \, (L)}_{k-1}\right), \,\,\, k=1, \ldots, L, \\
        h^{x, \,(L)}_0 &= x.
    \end{cases}
\end{equation}
The output of the network is $h_L^{x, \, (L)}$, which we denote by $\widehat{y}_L \hspace{-2pt} \left(x, W^{(L)} \right) $ to emphasize the dependence on the input $x$ and the weights $\parms^{(L)}$. \footnote{The analysis with bias is done by expanding the weights $\parms_k^{(L)}$ and the hidden states $h_k^{(L)}$ with an additional dimension.}  Fix a training set $D_N \coloneqq \{ (x_i,y_i) : i=1,\ldots,N \} \subset \R^d \times \R^d$, and the loss function $\ell\colon \R^d \times \R^d \to \R_+$ defined by $\ell(y, \widehat{y}) \coloneqq \frac{1}{2} \norm{y-\widehat{y}}_2^2$.

We study the dynamics of the weights induced by gradient descent (GD) on the mean-squared error $J_L \colon \R^{L \times d \times d} \to \R_{+}$ defined by 
\begin{equation} \label{mean_squared_error}
    J_L \hspace{-2pt} \left( \parms^{(L)} \right) \coloneqq \frac{1}{N} \sum_{i=1}^N \ell\left( y_i, \widehat{y}_L \hspace{-2pt} \left(x_i, \parms^{(L)} \right) \right) = \frac{1}{2N} \sum_{i=1}^N \norm{ y_i - \widehat{y}_L \hspace{-2pt} \left(x_i, \parms^{(L)}\right)}_2^2. 
\end{equation} 
We consider a gradient descent learning algorithm which sequentially updates the weights using an initialization $A^{(L)}(0)\in\R^{L\times d \times d}$ and
\begin{equation}  \label{eq:A_dot}
    \Delta A_k^{(L)}(t) \coloneqq  A_k^{(L)}(t+1)-  A_k^{(L)}(t) = - \lr_L(t) 
    \nabla_{\parms_k} J_L \hspace{-2pt} \left( A^{(L)}(t) \right),
\end{equation}
where $\lr_L(t)>0$ is the \textit{learning rate} at iteration $t\in\N$, which may depend on the depth $L$, but is independent of the layer index $k$.

\begin{assumption} \label{3_4_assumptions_1} \
There exists a constant $c_0 > 0$ such that
\begin{itemize}
\setlength\itemsep{0.3em}
\item[(i)] 
   Smooth activation function: $\sigma \in C^2(\R)$, $\sigma'(0)=1$ and for all $z\in\R$, $\abs{\sigma(z)} \leq \abs{z}$, $\abs{\sigma'(z)} \leq 1$ and $\abs{\sigma''(z)} \leq 1$.
    \item[(ii)]  Scaling factor: $\de = L^{-1/2}$. 
    \item[(iii)] Separated unit data: $\norm{x_i}_2 = \norm{y_i}_2 = 1$ and $\forall i\neq j$, $ \abs{ \langle x_i, x_j \rangle } \leq (8N)^{-1} e^{-4c_0}$.
    \item[(iv)] Initialisation with $O(1/L)$ weights:  \[
    \sup_{k,m} \norm{A_{k, m}^{(L)}(0)}_{2} \leq 2^{-9/2} N^{-1/2} d^{-1/2} e^{-4.2c_0} L^{-1}.
    \]
    \item[(v)] Small initial loss: $J_L(A^{(L)}(0)) \leq 2^{-15} 3^{-2} N^{-2} d^{-1} c_0^2 e^{-8.2c_0}$.
\end{itemize}
\end{assumption}
Note that $\tanh$ satisfies  Assumption \ref{3_4_assumptions_1} \textit{(i)}. Assumption \ref{3_4_assumptions_1} \textit{(ii)} comes from the scaling we observe in the experiments of Section \ref{sec:experiments-scaling}. Assumption \ref{3_4_assumptions_1} \textit{(iii)} requires the training points to be sufficiently orthogonal to one another. Among other cases, it is satisfied in the small data regime: take for example $N$ points uniformly at random on the $d-$dimensional sphere, where $d > N^4$. Hence, for $x_i \sim \UU( \Ss^{d-1} )$ i.i.d., we have by a union bound and Chebychev inequality:
\begin{align*}
    \Prob\left( \max_{i\neq j} \abs{ \langle x_i, x_j \rangle } > N^{-1} \right) &\leq N^2 \Prob\left( \abs{ \langle x_1, x_2 \rangle } > N^{-1} \right) \\
    &\leq N^4 \Var \left[ \langle x_1, x_2 \rangle \right] \leq N^4 \sum_{m=1}^d \E\left[ (x_1)_m^2 \right] \E\left[ (x_2)_m^2 \right] = N^4 d^{-1} < 1.
\end{align*}
Assumption \ref{3_4_assumptions_1} \textit{(iv)} guarantees that the network at initialization stay well-behaved, and does not bias the optimization path. Note also that Assumption \ref{3_4_assumptions_1} \textit{(iv)} does not rule out the case of a stochastic initialization. Assumption \ref{3_4_assumptions_1} \textit{(v)} relates to the fact that we are going to prove local convergence of gradient descent to zero training loss. Proving global convergence under our general framework is out of reach, as local minima are guaranteed to exist, see Theorem 2 in \cite{PS2021}. In this paper, we address Corollary 3 in \cite{PS2021} by providing conditions on the dataset and on the initialization procedure to show convergence of gradient descent for residual networks of large depth and finite width.
\section{Dynamics of weights and hidden states under gradient descent} \label{sec:results}

Recall that $\parms^{(L)}$ denotes a generic weight vector, whereas $A^{(L)}(t)$ denotes the weight vector obtained after $t$ iterations of gradient descent on the objective function $J_L$, where the initial weights $A^{(L)}(0)$ follow Assumption \ref{3_4_assumptions_1} \textit{(iv)}. The main results can be summarized as follows.

First, in Section \ref{sec:bounds-h-M-l}, we prove that if the network weights $\parms^{(L)}_k$ are $\OO(L^{-1/2})$,
then the hidden states $h_k^{x, \, (L)}$ and the Jacobian \begin{equation} \label{eq:def_jacobians}
    M_{k}^{x, \, (L)} \coloneqq \frac{\del h_L^{x, \, (L)}}{\del h_k^{x, \, (L)}} \in \R^{d\times d}
\end{equation} are uniformly bounded in $k$ and $L$. Then, under the same scaling assumption, we derive an upper bound for the norm of the gradient $\nabla_{\parms} J_L$ of the objective function with respect to the weights $\parms^{(L)}$.
Furthermore, we derive a lower bound for the norm of the gradient $\nabla_{\parms} J_L$
under the additional regularity assumption $\parms_{k+1}^{(L)} - \parms_{k}^{(L)} = \OO(L^{-1})$.

Next, in Section \ref{sec:behaviour-norms}, we let $\parms^{(L)}(0) \in \R^{L\times d \times d}$ be \textit{any} initialization and define recursively $\parms^{(L)}(t+1) = \parms^{(L)}(t) - \lr_L(t) \nabla_{\parms} J_L \hspace{-2pt} \left( \parms^{(L)}(t) \right) $. Under some scaling assumptions for $\parms^{(L)}(t)$ for $t=0, \ldots, T-1$, we show that the loss function $J_L \hspace{-2pt} \left( \parms^{(L)}(t) \right)$ at time $T$ admits an explicit upper bound. To show this, we study the effect of gradient descent on the following norms of the weight vector: 
\begin{equation} \label{eq:weight-norms}
\overbar{f}^{(L)} \left( \parms^{(L)}(t) \right) \coloneqq \frac{1}{2}\sum_{k=1}^L \norm{\parms^{(L)}_k(t)}_F^2 \quad \text{and} \quad \overbar{g}^{(L)} \left( \parms^{(L)}(t) \right) \coloneqq \frac{1}{2}L \sum_{k=1}^{L-1} \norm{\parms^{(L)}_{k+1}(t) - \parms_k^{(L)}(t)}_F^2.
\end{equation}
The scaling in $L$ is chosen in such a way that we will be able to prove a uniform bound (in $t$ and $L$) of the above norms along the gradient descent path $A^{(L)}(t)$ when $A^{(L)}(0)$ satisfy Assumption \ref{3_4_assumptions_1} \textit{(iii)}. 


Finally in Section \ref{sec:local-convergence} we show that under Assumption \ref{3_4_assumptions_1} with the parameter $A^{(L)}(t)$ evolving according to the gradient descent dynamics \eqref{eq:A_dot}, we have that for all $\epsilon > 0$, if we let $L = \Omega(1/\epsilon)$, $\lr_L(t) = \lr_0$, and $T_L^{\const} = \Theta(\eta_0^{-1} \log L) = \Omega(\eta_0^{-1} \log 1/\epsilon)$, then $J_L \hspace{-2pt} \left( A^{(L)}(T^{\const}_L) \right) < \epsilon$. That is, the loss function can be made arbitrarily small with practical values for the depth and the number of gradient steps. To prove this, we use recursion: we first verify the scaling assumptions 
\begin{equation} \label{eq:scaling-assumptions}
A^{(L)}(t) = \OO(c_0 L^{-1/2}) \quad \textrm{and} \quad A_{k+1}^{(L)}(t) - A_{k}^{(L)}(t) = \OO(e^{-4.2c_0} L^{-1}),
\end{equation}
at initialization, i.e. for $t=0$. This enables us to use the results of Section \ref{sec:bounds-h-M-l} to deduce an upper bound on the loss function $J_L \hspace{-2pt} \left( A^{(L)}(1) \right)$ at time $t=1$, which in turn yields that the scaling assumptions \eqref{eq:scaling-assumptions} are verified for $t=1$. We continue this process until the upper bound on the loss is smaller than $\epsilon$. \\
Further, we prove that for $T_L$ satisfying \eqref{eq:lr_bounds}, if the (pointwise) limit
\begin{equation} \label{scaling-limit}
\overbar{A}^*_s \coloneqq \lim_{L \to\infty} A^{(L)}_{\floor{Ls}}(T_L)
\end{equation}
converges uniformly in $s\in\left[0,1\right]$ at a $\OO(L^{-1/2})$ rate, then $\overbar{A}^*$ is of finite $2$-variation, giving an implicit regularity to the solution found by gradient descent. The numerical experiments in Section \ref{sec:numerical-experiments} confirm that these effects are observable in settings relevant to practical supervised learning problems.

\subsection{Bounds on the hidden states, their Jacobians, and the loss gradients} \label{sec:bounds-h-M-l}

We start the analysis by computing bounds on the hidden states and their Jacobians \eqref{eq:def_jacobians}. To do so, we define the following norm on the weights: 
\begin{equation} \label{def:norm-weights}
\norm{\parms^{(L)}}_{F, \infty} \coloneqq \max_{k=1, \ldots, L} \norm{\parms^{(L)}_k}_F,
\end{equation} 
where $\parms^{(L)}\in\R^{L\times d\times d}$ is a generic weight vector. We check that the hidden states are uniformly bounded from above and below in $k$ and $L$, and we prove an upper bound on the Jacobians, uniformly in $k$ and $L$. We get explicit bounds when $L$ is large enough:
\[
\norm{x}_2 e^{-2\Ca} \leq \norm{h^{x, \, (L)}_k}_2 \leq \norm{x}_2 e^{1.1\Ca} \quad \mbox{and} \quad \norm{M^{x, \, (L)}_{k} e_m}_2 \leq e^{\Ca},
\]
given the assumption that $\norm{\parms^{(L)}}_{F, \infty} \leq \Ca L^{-1/2}$. The proof can be found in Appendix \ref{app:proof_forward _backward}. Note that the bounds are deterministic, unlike the probabilistic results from \cite{ALS2019, ACGH2019}. Next, we derive that the norm of the gradient of the objective function is bounded above by $J_L^{1/2}$, so that it ensures that the gradient updates \eqref{eq:A_dot} stay local. The precise result and its proof can be found in Appendix \ref{app:gradient_estimation}. 

More crucially, we also need a lower bound on the norm of the gradient as a function of the \textit{suboptimality} gap. We first establish a lower bound for the gradient of the loss with respect to the weights of the first layer.

\begin{lemma} \label{lemma:gradient-lower}
Under Assumption \ref{3_4_assumptions_1} \textit{(i)}--\textit{(iii)}, let $\parms^{(L)}\in\R^{L\times d\times d}$ such that $L \geq \max(5c_0, 4c_0^2)$ and $\norm{\parms^{(L)}}_{F, \infty} \leq c_0 L^{-1/2}$ hold. Then,  we have
\begin{equation*}
\norm{\nabla_{\parms_1} J_L\left(\parms^{(L)} \right)}_F^2 \geq 
\frac{1}{4N} e^{-2c_0} L^{-1} J_L\left(\parms^{(L)} \right).
\end{equation*}
\end{lemma}

\begin{proof}
Fix $L \geq \max(5c_0, 4c_0^2)$. In the proof, we omit the explicit dependence in $L$. Observe first that  
\begin{align*}
    \norm{\nabla_{\parms_k} J_L(\parms)}_F^2 &= \sum_{m,n=1}^d \left( \frac{1}{N} \sum_{i=1}^N \frac{\del \ell}{\del \parms_{k, mn}}\left( y_i, \widehat{y}\left(x_i, \parms \right)  \right) \right)^2 \\
    &= \sum_{m,n=1}^d \frac{\deltal^2}{N^2} \sum_{i,j=1}^N h_{k-1,n}^{x_i} h_{k-1,n}^{x_j} \dot{\sigma}_{k,x_i,m} \dot{\sigma}_{k,x_j,m} \left( \left( M_{k}^{x_i} \right)^{\top} \left( \widehat{y}\left(x_i, \parms \right) - y_i \right) \right)_m \\
    &\hspace{1.5cm} \left( \left( M_{k}^{x_j} \right)^{\top} \left( \widehat{y}\left(x_j, \parms \right) - y_j \right) \right)_m \\
    &= \frac{\deltal^2}{N^2} \sum_{i,j=1}^N \big\langle h_{k-1}^{x_i}, \, h_{k-1}^{x_j} \big\rangle \, \widetilde{M}_{k, i, j},
\end{align*}
where \[
\widetilde{M}_{k, i, j} =  \Big\langle \dot{\sigma}_{k,x_i} \odot \left( M_{k}^{x_i} \right)^{\top} \left( \widehat{y}\left(x_i, \parms \right) - y_i \right), \,\, \dot{\sigma}_{k,x_j} \odot \left( M_{k}^{x_j} \right)^{\top} \left( \widehat{y}\left(x_j, \parms \right) - y_j \right) \Big\rangle.
\]
We focus on the case $k=1$. We first estimate, by Cauchy-Schwarz and Lemma \ref{lem:3_4_bounded_forward}, 
\begin{align*}
\abs{ \widetilde{M}_{1, i, j} } &\leq  \norm{ M_{1}^{x_i} }_2 \norm{ M_{1}^{x_j} }_2 \norm{ \widehat{y}\left(x_i, \parms \right) - y_i }_2 \norm{ \widehat{y}\left(x_j, \parms \right) - y_j }_2 \\
&\leq e^{2c_0} \norm{ \widehat{y}\left(x_i, \parms \right) - y_i }_2 \norm{ \widehat{y}\left(x_j, \parms \right) - y_j }_2. 
\end{align*}
\paragraph{Lower bound when $i=j$} First, as $\abs{\sigma''} \leq 1$ and $L \geq 4c_0^2$, we have $\dot{\sigma}_{1,x_i,m} = \sigma'(\parms_1 x_i)_m \geq 1 - \norm{\parms_1}_F \norm{x_i}_2 \geq 1 - c_0 L^{-1/2} \geq \frac{1}{2}$. Hence,
\begin{align*}
\widetilde{M}_{1, i, i} &= \norm{ \dot{\sigma}_{1,x_i} \odot \left( M_{1}^{x_i} \right)^{\top} \left( \widehat{y}\left(x_i, \parms \right) - y_i \right) }_2^2 \\
&\geq \frac{1}{4} \norm{\widehat{y}\left(x_i, \parms \right) - y_i}_2^2 \prod_{k=1}^L \left(1 - \deltal \norm{ \diag( \dot{\sigma}_{k, x_i}) \parms_k }_2 \right)^2 \\
&\geq \frac{1}{4} \left(1-\frac{c_0}{L} \right)^{2L} \norm{\widehat{y}\left(x_i, \parms \right) - y_i}_2^2 \geq \frac{1}{4} e^{-2c_0} \norm{\widehat{y}\left(x_i, \parms \right) - y_i}_2^2,
\end{align*}
where we applied Lemma \ref{app:lemma_product_matrix} in the second line, and the fact that $\norm{\cdot}_2 \leq \norm{\cdot}_F$. By Assumption \ref{3_4_assumptions_1} \textit{(iii)}, $\abs{ \langle x_i, x_j \rangle} \leq (8N)^{-1} e^{-4c_0}$ for all $i\neq j$, so we deduce 
\begin{align*}
    \norm{ \nabla_{\parms_1} J_L(\parms) }_F^2 &= \frac{\deltal^2}{N^2} \left( \sum_{i=1}^N \widetilde{M}_{1, i, i} \norm{x_i}_2^2 + \sum_{i\neq j} \widetilde{M}_{1, i, j} \langle x_i, x_j \rangle \right) \\
    &\geq \frac{1}{LN^2} \left( \frac{N}{2} e^{-2c_0} J_L(\parms)  - \frac{1}{8N} e^{-4c_0} \sum_{i\neq j} \abs{ \widetilde{M}_{1, i, j} } \right) \\
    &\geq \frac{1}{LN^2} \left( \frac{N}{2} e^{-2c_0} J_L(\parms)  - \frac{1}{8N} e^{-2c_0} \sum_{i\neq j} \norm{ \widehat{y}\left(x_i, \parms \right) - y_i }_2 \norm{ \widehat{y}\left(x_j, \parms \right) - y_j }_2   \right) \\
    &\geq \frac{1}{LN^2} \left( \frac{N}{2} e^{-2c_0} J_L(\parms)  - \frac{1}{8N} e^{-2c_0} \left( \sum_{i=1}^N \norm{ \widehat{y}\left(x_i, \parms \right) - y_i }_2 \right)^2 \right) \\
    &\geq \frac{1}{LN^2} \left( \frac{N}{2} e^{-2c_0} J_L(\parms)  - \frac{N}{4} e^{-2c_0} J_L(\parms) \right) = \frac{1}{4N} e^{-2c_0} L^{-1} J_L(\parms).
\end{align*}
\end{proof}

Next, if we assume that the weights $\parms^{(L)}$ are close to each other in neighbouring layers, we can deduce that the gradient of the loss with respect to weights in neighbouring layers are also close to each other. Hence, if we couple this fact with Lemma \ref{lemma:gradient-lower}, we can prove a lower bound on the norm of the gradient of the loss with respect to the full weight vector $\parms^{(L)}$. 

\begin{lemma} \label{lemma:lower-bound-grad}
Under Assumption \ref{3_4_assumptions_1} \textit{(i)}--\textit{(iii)}, let $\parms^{(L)}\in\R^{L\times d\times d}$ such that $L \geq \max(5c_0, 4c_0^2)$, $\norm{\parms^{(L)}}_{F, \infty} \leq c_0 L^{-1/2}$, and $\norm{\parms^{(L)}_{k+1} - \parms^{(L)}_{k}}_{F} \leq 2^{-7/2} N^{-1/2} e^{-4.2c_0} L^{-1}$ for each $k$.  Then, \[
\norm{\nabla_{\parms^{(L)}} J_L \hspace{-2pt} \left( \parms^{(L)} \right)}_F^2 \geq \left( \frac{1}{16} N^{-1} e^{-2c_0}  - 17 d c_0^4 e^{6.4c_0} L^{-1} \right) J_L(\parms).
\]
\end{lemma}
\vspace{0.2cm}
\begin{proof}
Fix $L \geq \max(5c_0, 4c_0^2)$. In the proof, we omit the explicit dependence in $L$. We use Lemma \ref{lem:neighbour-grad} to estimate the difference of neighbouring gradients:
\begin{align*}
\frac{\del J_L}{\del \parms_{k, mn}} - \frac{\del J_L}{\del \parms_{k+1, mn}} &= \frac{\deltal}{N} \sum_{i=1}^N h_{k-1, n}^{x_i} \left( \dot{\sigma}_{k, x_i, m} - \dot{\sigma}_{k+1, x_i, m} \right) \nabla_{\widehat{y}} \, \ell\left(y_i, \widehat{y}(x_i, \parms) \right)^{\top} M_{k+1}^{x_i} e_m  \\
    &+ \frac{\deltal^2}{N} \sum_{i=1}^N \nabla_{\widehat{y}} \, \ell\left(y_i, \widehat{y}(x_i, \parms) \right)^{\top} M_{k+1}^{x_i} \xi^{x_i, \, (L)}_{k, mn},
\end{align*}
where $\xi^{x, \, (L)}_{k, mn}$ satisfies
\[
\norm{ \xi^{x, \, (L)}_{k, mn} }_2^2 \leq 2\left( h^x_{k-1,n} \right)^2 \norm{\parms_{k+1} - \parms_k}_F^2 + 2 \norm{\parms_{k,n}}_2^4 \norm{h_{k-1}^x}_2^4.
\]
By Lemma \ref{lem:3_4_bounded_forward} and the fact that $\sigma'$ is $1-$Lipschitz by Assumption \ref{3_4_assumptions_1} \textit{(i)}, we bound further:
\begin{align*}
&\norm{ \nabla_{\parms_{k+1}} J_L\left( \parms \right) - \nabla_{\parms_{k}} J_L\left( \parms \right) }_F^2 = \sum_{m,n=1}^d \left( \frac{\del J_L}{\del \parms_{k, mn}} - \frac{\del J_L}{\del \parms_{k+1, mn}} \right)^2 \\
&\leq 4 \sum_{m,n=1}^d \frac{1}{LN} \sum_{i=1}^N \left( h_{k-1, n}^{x_i} \right)^2 \norm{M^{x_i}_{k} e_m}_2^2 \left( \parms_k h^{x_i}_{k-1} - \parms_{k+1} h^{x_i}_k \right)_m^2 \ell(y_i, \widehat{y}(x_i, \parms)) \\
&\qquad + \frac{4}{L^2N} \sum_{i=1}^N e^{2c_0} \left( 2d e^{2.2c_0} \norm{\parms_{k+1} - \parms_k}_F^2 + 2d c_0^4 e^{4.4c_0} L^{-2} \right) \ell(y_i, \widehat{y}(x_i, \parms)) \\
&\leq  e^{4.2c_0} \frac{4}{LN}\sum_{i=1}^N \norm{ \parms_k h^{x_i}_{k-1} - \parms_{k+1} h^{x_i}_k }_2^2 \ell(y_i, \widehat{y}(x_i, \parms)) + 9d c_0^4 e^{6.4c_0} L^{-4} J_L(\parms).
\end{align*}
Then, simply note that 
\begin{align*}
\norm{ \parms_k h^{x_i}_{k-1} - \parms_{k+1} h^{x_i}_k }_2^2 &\leq 2\norm{ \left( \parms_{k+1} - \parms_{k} \right) h^{x_i}_k}_2^2 +  2\norm{ \parms_k \left(h^{x_i}_{k} - h^{x_i}_{k-1} \right)}_2^2 \\
&\leq \frac{1}{64} N^{-1} e^{-6.2c_0}  L^{-2} + 2c_0^4 e^{2.2c_0} L^{-3}.  
\end{align*}
Hence, 
\begin{equation*}
    \norm{ \nabla_{\parms_{k+1}} J_L\left( \parms \right) - \nabla_{\parms_{k}} J_L\left( \parms \right) }_F^2 \leq \left( \frac{1}{16} N^{-1} e^{-2c_0}  + 17d c_0^4 e^{6.4c_0} L^{-1} \right) L^{-3} J_L(\parms). 
\end{equation*}
Finally, we use the reverse triangle inequality and Cauchy-Schwarz inequality:
\begin{align*}
    \norm{ \nabla_{\parms_k} J_L\left( \parms \right) }_F^2 &\geq \frac{1}{2} \norm{ \nabla_{\parms_1} J_L\left( \parms \right) }_F^2 - (k-1) \sum_{k'=1}^{k-1} \norm{ \nabla_{\parms_{k+1}} J_L\left( \parms \right) - \nabla_{\parms_{k}} J_L\left( \parms \right)}_F^2 \\
    &\geq \frac{1}{8} N^{-1} e^{-2c_0} L^{-1} J_L(\parms) - \frac{(k-1)^2}{L^3} \left( \frac{1}{16} N^{-1} e^{-2c_0} + 17d c_0^4 e^{6.4c_0} L^{-1} \right) J_L(\parms)  \\
    &\geq \left( \frac{1}{16} N^{-1} e^{-2c_0} - 17d c_0^4 e^{6.4c_0} L^{-1} \right) L^{-1} J_L(\parms).
\end{align*}
The second inequality holds by Lemma \ref{lemma:gradient-lower} and \textit{(i)} above. Hence, \[
\norm{\nabla_{\parms} J_L \hspace{-2pt} \left( \parms^{(L)}\right) }_F^2 = \sum_{k=1}^L \norm{\nabla_{\parms_k} J_L(\parms^{(L)})}_F^2 \geq \left( \frac{1}{16} N^{-1} e^{-2c_0}  - 17d c_0^4 e^{6.4c_0} L^{-1} \right) J_L(\parms).
\]
\end{proof}

It guarantees that for $L\gg 1$, every critical point close to the origin is a global minimum of the objective function, similarly to what is known for linear residual networks \cite{HM2017, K2016, LB2018, LK2017}. 


\subsection{Behaviour of weight norms along the gradient descent path} \label{sec:behaviour-norms}


In Section \ref{sec:bounds-h-M-l}, we establish bounds on the gradient of the loss function evaluated at a generic weight vector $\parms^{(L)}\in\R^{L\times d\times d}$. We now proceed to understand how $\parms^{(L)}$ changes under a gradient descent update. To do so, we study the local version of the weight norms defined in \eqref{eq:weight-norms}.
Define for $x,y\in\R^d$ and $k = 0, \ldots, L$: 
\begin{equation} \label{def_G}
G^{\, x, y, \, (L)}_{k} \hspace{-2pt} \left( \parms^{(L)} \right) \coloneqq \frac{\del \ell(y, \,\cdot\,)}{\del h^{(L)}_k}\left( \widehat{y}\big( x, \parms^{(L)} \big) \right) \in \R^d.
\end{equation}
Also, for clarity, denote $h_k^{x, (L)} \hspace{-2pt} \left( \parms^{(L)} \right) \in \R^d$ for the hidden state of the $k^{th}$ layer using input $x\in\R^d$ and network weights $\parms^{(L)}\in\R^{L\times d \times d}$.

\begin{lemma} \label{lemma-f}
Let $\parms^{(L)}\in\R^{L\times d\times d}$ and define $\widetilde{\parms}^{(L)} \coloneqq \parms^{(L)} - \lr_L \nabla_{\parms} J_L\left( \parms^{(L)} \right)$. Define further $f^{(L)}_{k,m}\left( \parms^{(L)} \right) \coloneqq \frac{1}{2} L \norm{\parms^{(L)}_{k, m}}_2^2$  Under Assumption \ref{3_4_assumptions_1} \textit{(i)}--\textit{(ii)}, we have
\begin{equation*}
    f^{(L)}_{k,m}\left(\widetilde{\parms}^{(L)}\right)^{1/2} \leq f^{(L)}_{k,m}\left(\parms^{(L)}\right)^{1/2} + \frac{1}{\sqrt{2}} \lr_L \left( \frac{1}{N} \sum_{i=1}^N \norm{h_{k-1}^{x_i,\,(L)} \hspace{-2pt} \left( \parms^{(L)} \right) }_2^2 \norm{ G^{\, x_i, y_i, \, (L)}_{k}\hspace{-2pt}\left( \parms^{(L)} \right) }_{\infty}^2 \right)^{1/2}.
\end{equation*}
\end{lemma}


\begin{lemma} \label{lemma-g}
Let $\parms^{(L)}\in\R^{L\times d\times d}$ and $\Ca > 0$ such that $L \geq 5 \Ca$ and $\norm{\parms^{(L)}}_{F, \infty} \leq \Ca L^{-1/2}$. Define $\widetilde{\parms}^{(L)} \coloneqq \parms^{(L)} - \lr_L \nabla_{\parms} J_L\left( \parms^{(L)} \right)$, and let $g^{(L)}_{k}\left( \parms^{(L)} \right) \coloneqq \frac{1}{2} L^2 \norm{ \parms^{(L)}_{k+1} - \parms^{(L)}_{k} }_F^2$. Under Assumption \ref{3_4_assumptions_1} \textit{(i)}--\textit{(ii)}, we have
\begin{align*}
 g^{(L)}_{k}\left(\widetilde{\parms}^{(L)}\right) &\leq g^{(L)}_{k} \left(\parms^{(L)}\right) \left( 1 + L^{-1/2} \lr_L \frac{1}{N} \sum_{i=1}^N \norm{h_{k-1}^{x_i,\,(L)}\left(\parms^{(L)}\right)}_2^2 \norm{ G^{x_i, y_i,\,(L)}_{k+1}\left(\parms^{(L)}\right) }_{\infty} \right)^2  \\
 & + \OO\Big( \Ca e^{2.1\Ca} \left( \Ca e^{1.1\Ca} L^{-1} + 2 L^{-1/2} \right) \lr_L g_{k}\big(\parms^{(L)}\big)^{1/2}  J_L\big(\parms^{(L)}\big)^{1/2} \Big),
\end{align*}
where the Big-O constant is also independent of $\Ca$.
\end{lemma}
The proofs of Lemmas \ref{lemma-f} and \ref{lemma-g} can be found in Appendix \ref{app:behaviour-norms}. 


\subsection{Local convergence of gradient descent}\label{sec:local-convergence}

In this section, we initialize the weight vector $A^{(L)}(0)$ according to Assumption \ref{3_4_assumptions_1} \textit{(iv)} and we let the weights $A^{(L)}(t)$ evolve according to the gradient descent dynamics \eqref{eq:A_dot}. We show that under some \textit{a priori} conditions on the initial parameters, the initial loss, and the learning rates, we are able to prove a practical upper bound on the loss function along the gradient descent path.

\begin{theorem} \label{thm:global-convergence}
Let $L$ be large enough. Under Assumption \ref{3_4_assumptions_1} \textit{(i)}--\textit{(v)}, let the parameter $A^{(L)}(t)$ evolve according to the gradient descent dynamics \eqref{eq:A_dot} with learning rates $\lr_L(t)$ until time $T_L \in \N$, chosen in such a way that for each $t = 0, \ldots, T_L-1$, we have 
\begin{equation} \label{eq:lr_bounds}
    \lr_L(t) \leq \frac{1}{160} N^{-1} d^{-1} e^{-10.5c_0} \quad \mbox{and} \quad \sum_{t=0}^{T_L-1} \lr_L(t) \leq d^{-1} \log L.
\end{equation}
Then, for each $t = 0, \ldots, T_L$, we have
\[
J_L(A(t)) \leq \exp\left( - \frac{1}{32} N^{-1} e^{-2c_0} \sum_{t'=0}^{t-1} \eta_L(t') \right) J_0 + 34 d c_0^4 e^{6.4c_0} \left( \sum_{t'=0}^{t-1} \lr_L(t') \right) L^{-1} J_0.
\]
\end{theorem}

Theorem \ref{thm:global-convergence} is a local convergence result since we assume that the initial loss lies below a certain level by Assumption \ref{3_4_assumptions_1} \textit{(v)}. We are able to show convergence as $L\to\infty$ of the loss to zero when the horizon $T_L$ depends explicitly on the depth while satisfying \eqref{eq:lr_bounds}.

\begin{proof}
We choose $L$ big enough so that 
\begin{equation} \label{app:eq:bound_L}
    \frac{3}{64} N^{-1} d^{-1} c_0^2 e^{2.2c_0} (\log L)^{3/2} \leq L^{1/2}, \quad 34c_0^4 e^{6.4c_0} \log L \leq L.
\end{equation}
Note that it trivially implies that $L \geq \max(4c_0^2, 5c_0)$. In the proof, we omit the explicit dependence in $L$. Denote $J_0 \coloneqq J_L(A(0))$ the initial loss. We first prove jointly that 
\begin{equation} \label{induction-hyp}
    \begin{aligned}
        J_L(A(t)) &\leq 2 J_0, \\
        \max_k \norm{A_k(t)}_{F} &\leq c_0 L^{-1/2}, \\
        \max_{k} \norm{A_{k+1}(t) - A_k(t)}_F &\leq 2^{-7/2} N^{-1/2} e^{-4.2c_0} L^{-1}. \\
    \end{aligned}
\end{equation}
for $t=0, \ldots, T_L$ by induction on $t$. For $t=0$, by Assumption \ref{3_4_assumptions_1} \textit{(iv)}, we directly have 
\begin{equation} \label{eq:bound-A0}
    \begin{aligned}
        \max_k \norm{A_k(0)}_{F} &\leq d^{1/2} \sup_{k,m} \norm{A_{k, m}(0)}_{2} \leq L^{-1} < c_0 L^{-1/2}, \\
        \norm{A_{k+1}(0) - A_{k}(0)}_F &\leq d^{1/2}  \sup_{m} \norm{A_{k+1, m}(0) - A_{k, m}(0)}_{2} < 2^{-7/2} N^{-1/2} e^{-4.2c_0}  L^{-1}.
    \end{aligned}
\end{equation}
Let $t\geq 0$. Assume that \eqref{induction-hyp} holds true for all $t'\leq t < T_L$. We prove that \eqref{induction-hyp} holds for $t+1$. Define $f_{k,m}(t) \coloneqq f_{k,m}\left(A^{(L)}(t)\right)$ as in Lemma \ref{lemma-f} and $g_{k}(t) \coloneqq g_{k}\left( A^{(L)}(t) \right)$ as in Lemma \ref{lemma-g}. As $L \geq \max(4 c_0^2, 5c_0)$, we can apply Lemma \ref{lemma-f} and Lemma \ref{lem:3_4_bounded_forward} with the induction hypothesis.
\begin{align}
    f_{k,m}(t+1)^{1/2} &\leq f_{k,m}(t)^{1/2} + \frac{1}{\sqrt{2}} e^{2.1 c_0} \lr_L(t) \left( \frac{2}{N} \sum_{i=1}^N  \ell\left(y_i, \widehat{y}(x_i, A(t)) \right) \right)^{1/2} \nonumber \\
    &= f_{k,m}(t)^{1/2} + e^{2.1 c_0} \lr_L(t) J_L(A(t))^{1/2}. \label{eq:recurrence-f}
\end{align}
Similarly, we apply Lemma \ref{lemma-g} with $\Ca = c_0$ and Lemma \ref{lem:3_4_bounded_forward} with the induction hypothesis.
\begin{align} 
    g_{k}(t+1) &\leq g_{k}(t) \left( 1 + e^{3.2 c_0} \lr_L(t) L^{-1/2} J_L(A(t))^{1/2} \right)^2 \nonumber \\
    &\quad + \OO\Big( c_0 e^{2.1 c_0} \left( c_0 e^{1.1 c_0} L^{-1} + 2 L^{-1/2} \right) \lr_L g_{k}(t)^{1/2}  J_L\big(A(t)\big)^{1/2} \Big). \label{eq-recurrence-g}
\end{align}
Now, we want to apply Lemma \ref{bound-loss-function} to bound $J_L(A(t))$. We check that using Lemma \ref{lemma:lower-bound-grad}, the assumptions of Lemma \ref{bound-loss-function} are verified for 
\begin{equation*}
\Ca(t')=c_0, \quad \underline{c} \equiv \underline{c}(t') = \frac{1}{16} N^{-1} e^{-2c_0}, \quad \overbar{c} \equiv \overbar{c}(t') = 34dc_0^4 e^{6.4c_0} J_0 \quad \text{for} \,\, t'\leq t.
\end{equation*}
Thus, as $\eta_L(t) < 2^{-5} 5^{-1} N^{-1} d^{-1} e^{-10.5 c_0} < 2^{-1} c_0 e^{-3.2 c_0}$, we deduce the following bound on the loss function at all times $t'=0, \ldots, t+1$.
\begin{equation} \label{eq:bound-loss-fct}
J_L(A(t')) \leq \exp\left( -\frac{1}{2} \underbar{c}  \sum_{t''=0}^{t'-1} \lr_L(t'')  \right) J_0 + \overbar{c}  L^{-1} \sum_{t''=0}^{t'-1} \lr_L(t'') 
\end{equation}
\underline{Bound on $J_L(A(t+1))$}: Plugging in \eqref{eq:lr_bounds} and \eqref{app:eq:bound_L} into \eqref{eq:bound-loss-fct}, we verify that
\begin{equation*}
J_L(A(t+1)) \leq \left(1 + 34 c_0^4 e^{6.4c_0} L^{-1} \log L \right) J_0 \leq 2 J_L(A(0)).
\end{equation*}
\underline{Bound on $f_{k,m}(t+1)$}:  We plug \eqref{eq:bound-loss-fct} into \eqref{eq:recurrence-f} and sum over $t$ to deduce
\begin{align} 
f_{k,m}(t+1)^{1/2} &\leq f_{k,m}(0)^{1/2} +  e^{2.1 c_0} \sum_{t'=0}^{t} \lr_L(t') J_L(A(t'))^{1/2} \nonumber \\
&\leq \frac{1}{3\sqrt{2}} d^{-1/2} c_0 L^{-1/2} + e^{2.1 c_0} R_L(t), \label{eq:recurrence-f-2}
\end{align}
where we use \eqref{eq:bound-A0} for the second inequality and \[
R_L(t) \coloneqq \sum_{t'=0}^{t} \eta_L(t') J_L(A(t'))^{1/2}.
\]
To find an upper bound to $R_L(t)$, we use the inequality $\sqrt{x+y} \leq \sqrt{x} + \sqrt{y}$ in \eqref{eq:bound-loss-fct}, with the help of \eqref{app:eq:bound_L}:
\begin{align*}
    R_L(t) &\leq \sum_{t'=0}^{t} \lr_L(t') \exp\left( - \frac{1}{4} \underbar{c}  \sum_{t''=0}^{t'-1} \lr_L(t'')  \right) J_0^{1/2} + \overbar{c}^{1/2} L^{-1/2} \sum_{t'=0}^{t} \lr_L(t') \left( \sum_{t''=0}^{t'-1} \lr_L(t'') \right)^{1/2} 
\end{align*}
Now, we estimate the following quantity using \eqref{eq:lr_bounds}:
\begin{align*}
    \sum_{t'=0}^{t} \lr_L(t') \left( \sum_{t''=0}^{t'-1} \lr_L(t'') \right)^{1/2} &\leq \left( \sum_{t'=0}^t \eta_L(t') \right)^{3/2} \leq d^{-3/2} (\log L)^{3/2}. 
\end{align*}
Next, we use the fact $\lr_L(t) < \overbar{\lr} = 2^{-5} 5^{-1} N^{-1} d^{-1} e^{-10.5c_0} $ and \[
\left( 1 - \exp\left( -\frac{1}{4} \underbar{c} \overbar{\lr} \right) \right) x \leq \overbar{\lr} \left( 1 - \exp\left(-\frac{1}{4}\underbar{c} x \right) \right)  \,\,\, \mbox{for all} \,\, x\in\left[0, \overbar{\lr} \right]
\]
to deduce that the following sum is telescoping:
\begin{align*}
\sum_{t'=0}^{t} \lr_L(t') \exp\left( - \frac{1}{4} \underbar{c} \sum_{t''=0}^{t'-1} \lr_L(t'')  \right) &\leq \frac{ 1 - \exp\left( -\frac{1}{4} \underbar{c} \sum_{t'=0}^{t} \lr_L(t') \right)}{1 - \exp(-\frac{1}{4} \underbar{c} \overbar{\lr})} \overbar{\lr} \\
&\leq 8 \underbar{c}^{-1}.
\end{align*}
Hence, by Assumption \ref{3_4_assumptions_1} \textit{(v)} and \eqref{app:eq:bound_L},
\begin{align}
R_L(t) &\leq 128N e^{2c_0} J_0^{1/2} + 6 d^{1/2} c_0^2 e^{3.2c_0} (\log L)^{3/2} L^{-1/2}  J_0^{1/2} \nonumber \\
&\leq \frac{2c_0}{3\sqrt{2}} d^{-1/2} e^{-2.1c_0}  \label{eq:bound-rl}
\end{align}
Plugging it in \eqref{eq:recurrence-f-2}, we obtain \[
f_{k,m}(t+1)^{1/2} \leq \frac{c_0}{3\sqrt{2}} d^{-1/2} + \frac{2c_0}{3\sqrt{2}} d^{-1/2}  = \frac{c_0}{\sqrt{2}} d^{-1/2}.
\]
Hence, this completes the induction step for the norm of $A$:
\[
\norm{A^{(L)}(t+1)}_{F, \infty} \leq \sqrt{2} d^{1/2}  L^{-1/2}  \sup_{k,m} f_{k,m}(t+1)^{1/2} \leq c_0 L^{-1/2}.
\]
\underline{Bound on $g_{k}(t+1)$}:
By \eqref{app:eq:bound_L}, $L^{1/2} \geq c_0 e^{1.1c_0}$, so we can rewrite \eqref{eq-recurrence-g}:
\[
    g_{k}(t+1) \leq g_{k}(t) u_L(t)  + g_{k}(t)^{1/2} \, \OO\left( c_0 e^{2.1c_0} \lr_L(t) L^{-1/2} J_L(A(t))^{1/2} \right), 
\]
where 
\begin{equation*}
u_L(t) \coloneqq \left(1 + \lr_L(t) L^{-1/2} e^{3.2 c_0} J_L(A(t))^{1/2} \right)^2.
\end{equation*}
We can thus apply Lemma \ref{lemma:gronwall} \textit{(ii)}, together with the identity $1+x\leq \exp(x)$ and \eqref{eq:bound-rl} to deduce that 
\begin{align*}
    g_{k}(t+1)^{1/2} &\leq \exp\left(  e^{3.2c_0} R_L(t) L^{-1/2} \right) \left(  g_{k}(0)^{1/2} + \OO\left(c_0 e^{2.1c_0} R_L(t) L^{-1/2} \right) \right) \\
    &\leq \exp\left( c_0 e^{1.1c_0} L^{-1/2} \right) \left(  1 + \OO( c_0^2 L^{-1/2}) \right) g_{k}(0)^{1/2} \\
    &\leq \sqrt{2} g_{k}(0)^{1/2}.
\end{align*}
The last inequality is derived with the help of \eqref{app:eq:bound_L}. We finish the induction step by observing that $g_{k,m}(0)^{1/2} \leq 2^{-9/2} N^{-1/2} d^{-1/2} e^{-4.2c_0} $ by Assumption \ref{3_4_assumptions_1} \textit{(iv)}. \\

\noindent \underline{Convergence of $J_L(A(T_L)) \to 0$}: We now have all the tools to deduce the rate of convergence of $J_L(A(T_L))$ to zero. We observe from the induction result above that the assumptions of Lemma \ref{bound-loss-function} are verified for $\Ca(t) = c_0$, $\underline{c}(t) = 2^{-4} N^{-1} e^{-2c_0}$ and $\overbar{c}(t) = 34 d c_0^4 e^{6.4c_0} J_0$ by Lemma \ref{lemma:lower-bound-grad}, for each $t\in\left[0, T\right)$. In particular, we have  
\begin{equation}
J_L(A(T_L)) \leq \exp\left( - \frac{1}{32} N^{-1} e^{-2c_0} \sum_{t=0}^{T_L-1} \eta_L(t) \right) J_0 + 34 d c_0^4 e^{6.4c_0} \left( \sum_{t=0}^{T_L-1} \lr_L(t) \right) L^{-1} J_0.
\end{equation}

\end{proof}

\begin{remark} \label{rmk:learning-rates}
Let $\lr_0 > 0$ be a fixed learning rate, independent of $k,L$ and $t$, and let $J_0 \coloneqq J_L\hspace{-2pt} \left(A^{(L)}(0)\right)$ be the initial loss. Observe that from Theorem~\ref{thm:global-convergence}, if we choose 
\begin{itemize}
    \item $\lr_L(t) = \lr_0$ and $T_L^{\const} = \Theta(\lr_0^{-1} \log L)$, then conditions \eqref{eq:lr_bounds} are satisfied, so we deduce \[
    J_L\hspace{-2pt} \left(A^{(L)}(T_L^{\const})\right) \leq \exp(-\underline{c}\lr_0 T_L^{\const}) J_0 + \OO(\lr_0 T_L^{\const} L^{-1}).
    \]
    Hence, for an error level $\epsilon > 0$, gradient descent with constant learning rate for a network of depth $L = \Omega(1/\epsilon)$ reaches $J_L\hspace{-2pt}\left(A^{(L)}(T_L^{\const})\right) < \epsilon$ in $\Theta( \lr_0^{-1} \log 1/\epsilon)$ iterations. 
    \item $\lr_L(t) = \lr_0 (t+1)^{-1}$ and $T_L^{\decay} = \Theta(\exp(\lr_0^{-1} \log L))$, then conditions \eqref{eq:lr_bounds} are satisfied. We deduce that \[
    J_L\hspace{-2pt} \left(A^{(L)}(T_L^{\decay})\right) \leq \exp(-\underline{c}\lr_0 \log T_L^{\decay}) J_0 + \OO(\lr_0 \log T_L^{\decay} L^{-1} ).
    \]
    Hence, for an error level $\epsilon > 0$, gradient descent with decaying learning rate for a network of depth $L = \Omega(1/\epsilon)$ reaches $J_L\hspace{-2pt}\left(A^{(L)}(T_L^{\decay})\right) < \epsilon$ in $\Theta(\exp( \lr_0^{-1} \log 1/\epsilon))$ iterations.
\end{itemize}
\end{remark}
The above convergence rates above are confirmed by our experiments in Section \ref{sec:numerical-experiments}. Note that gradient descent converges exponentially faster when using constant learning rates rather than decaying ones. This is because the parameters $A^{(L)}(t)$ and the gradients $\nabla_A J_L(A^{(L)}(t))$ are already on the same scale $\OO(L^{-1/2})$. Note also that Theorem~\ref{thm:global-convergence} is not in contradiction with \cite[Theorem 6]{BEL2018} stating that gradient descent might get stuck at the critical point $\left(\de, A^{(L)}\right) = (0,0)$ that is usually not a global minimizer. Indeed, we force $\de$ to have a non-trivial scaling by Assumption \ref{3_4_assumptions_1} \textit{(iv)}, so that $(0,0)$ is simply not a point in the parameter space.

\subsection{Scaling limit of  trained weights}
In many cases the trained weights, viewed as a function of the layer index $k/L$, have a scaling limit which is a function defined on $[0,1]$. 
We   show that such a limit then admits  finite $p$-variation with $p=2$.
\begin{proposition} \label{prop:tanh_continuity}
Let $\left( A^{(L)}(t) \colon t=1, \ldots, T_L\right)$ follow the gradient descent dynamics \eqref{eq:A_dot}, where the assumptions of Theorem \ref{thm:global-convergence} are satisfied for $T=T_L$. Assume  there exists $\overbar{A}^* \coloneqq \left[0,1\right] \to \R^{d\times d}$ such that
\begin{equation} \label{ass:scaling-limit}
    \begin{aligned}
        \mathop{\sup}_{s\in [0,1]}L^{1/2} \norm{ L^{1/2} A_{\floor{Ls}}^{(L)}(T_L) - \overbar{A}^*_{s}}_F &\mathop{\longrightarrow}^{L\to\infty} 0
    \end{aligned}
\end{equation}
Then, the scaling limit $\overbar{A}^*$ has finite $p$-variation with $p=2$.
\end{proposition}

Conditions \ref{ass:scaling-limit} may seem strong, but they are related to the norm $f^{(L)}_{k,m}(T_L)$ defined in Lemma \ref{lemma-f} having a limit as $L\to\infty$. Under the hypothesis of Theorem \ref{thm:global-convergence}, we have shown in the proof of Theorem \ref{thm:global-convergence} that the norm $f^{(L)}_{\floor{Ls}, m}(T_L)$ stay uniformly bounded (in $s$ and $m$) as $L\to\infty$. Condition \ref{ass:scaling-limit} has also been verified in numerical experiments, see Section \ref{sec:regularity-limit}.

\begin{proof}
Fix a partition $\pi = \left\{ 0 = s_0 < s_1 < \ldots < s_K = 1 \right\}$, where the mesh of the partition $\norm{\pi}$ is small enough. In the following, $c>0$ denotes a constant independant of $s$ and $L$. For $i=1, \ldots, K-1$, let $L_i \in \N$ big enough so that Theorem \ref{thm:global-convergence} applies. We estimate directly 
\begin{align*}
\norm{ \overbar{A}^*_{s_{i+1}} - \overbar{A}^*_{s_i}}_F &\leq \norm{ L_i^{1/2} A_{\floor{L_i s_{i+1}}}^{(L_i)}(T_{L_i}) - \overbar{A}^*_{s_{i+1}}}_F + \norm{ L_i^{1/2} A_{\floor{L_i s_i}}^{(L_i)}(T_{L_i}) - \overbar{A}^*_{s_i}}_F \\
&+ L_i^{1/2} \norm{ A^{(L_i)}_{\floor{L_i s_{i+1}}} (T_{L_i}) - A^{(L_i)}_{\floor{L_i s_{i}}}(T_{L_i})}_F \\
&\leq c L_i^{-1/2} + L_i^{1/2} \norm{ A^{(L_i)}_{\floor{L_i s_{i+1}}} (T_{L_i}) - A^{(L_i)}_{\floor{L_i s_{i}}}(T_{L_i})}_F
\end{align*}
We now use the proof of Theorem \ref{thm:global-convergence} to deduce a uniform bound (in $k$ and $L$) on the quantity $g^{(L)}_{k}(T_L)$ defined in Lemma \ref{lemma-g}. That means, $L \norm{ A^{(L)}_{\floor{Ls}} (T_L) - A^{(L)}_{\floor{Ls}+1}(T_L)}_F < c < \infty$. We can apply the triangle inequality to deduce 
\begin{equation*}
\norm{ \overbar{A}^*_{s_{i+1}} - \overbar{A}^*_{s_i}}_F \leq c L_i^{-1/2} + c L_i^{-1/2} \left( \floor{L_i s_{i+1}} - \floor{L_i s_i} \right) \leq c L_i^{-1/2} + c L_i^{1/2} \abs{s_{i+1} - s_i}.  
\end{equation*}
Hence,
\begin{equation} \label{eq:bound-qv}
\sum_{i=0}^{K-1} \norm{ \overbar{A}^*_{s_{i+1}} - \overbar{A}^*_{s_i}}_F^2 \leq c \sum_{i=0}^{K-1} L_i^{-1} + L_i \abs{s_{i+1} - s_i}^2.  
\end{equation}
As $\norm{\pi}$ is small enough, we can choose $L_i = \Theta(\abs{s_{i+1} - s_i}^{-1})$ to deduce that the RHS of \eqref{eq:bound-qv} is bounded uniformly in $\pi$. Taking a supremum over all such partitions  then show that $\overbar{A}^*$ has finite $p$-variation with $p=2$.
\end{proof}

\section{Numerical experiments} \label{sec:numerical-experiments}

To illustrate the results of Section \ref{sec:results}, we design numerical experiments with the following set-up. We have a fixed training set $\left\{ (x_i, y_i) : i=1, \ldots, N \right\}$ in $\R^d \times \R^d$, where $d$ is the dimension of the inputs and outputs and $N$ is the size of the dataset. For any depth $L\in\N$, we initialize the weights of the network \eqref{resnets_wo_bias} with $\deltal = L^{-\alpha_0}$ and each entry of $A_k^{(L)}$ is independent and normally distributed with standard deviation $d^{-1}L^{-\beta_0}$, where $\alpha_0, \beta_0 \in \left[0,1\right]$. The weights are trained using gradient descent on the (unregularized) mean squared error $J_L$ defined in \eqref{mean_squared_error} with a fixed learning rate $\eta_0$ independent of $d,k,L$ and the training time $t$. We perform a fixed number $T\in\N$ of gradient updates, with no early stopping. 

\subsection{Identification of scaling behavior} \label{sec:experiments-scaling}
We run two experiments to discover the best scaling for $\deltal$. Denote $\alpha_t$ the scaling of $\deltal$ at time $t$, i.e. $\alpha_t \propto L^{-\alpha_t}$, and denote $\beta_t$ the scaling of the weights $A^{(L)}(t)$ at time $t$, i.e. $A^{(L)}(t) \propto L^{-\beta_t}$.
The first experiment is to let $\deltal$ trainable with gradient descent with learning rate $\eta_0$, and observe the resulting scaling $\alpha_t$.
\begin{figure}[!ht]
    \centering
    \includegraphics[width=0.48\textwidth]{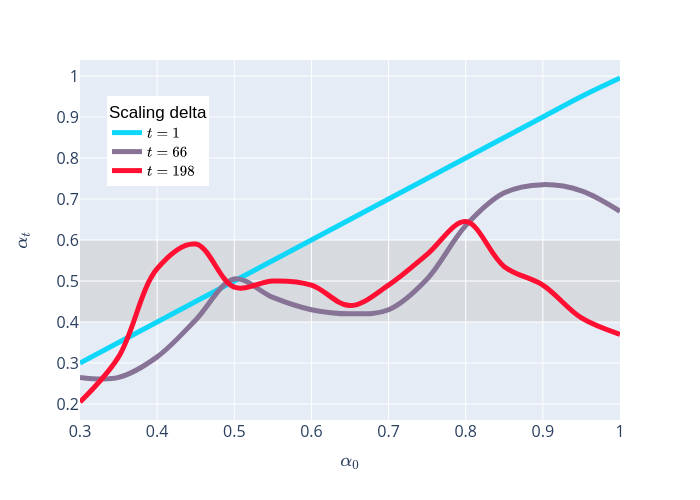}
    \includegraphics[width=0.48\textwidth]{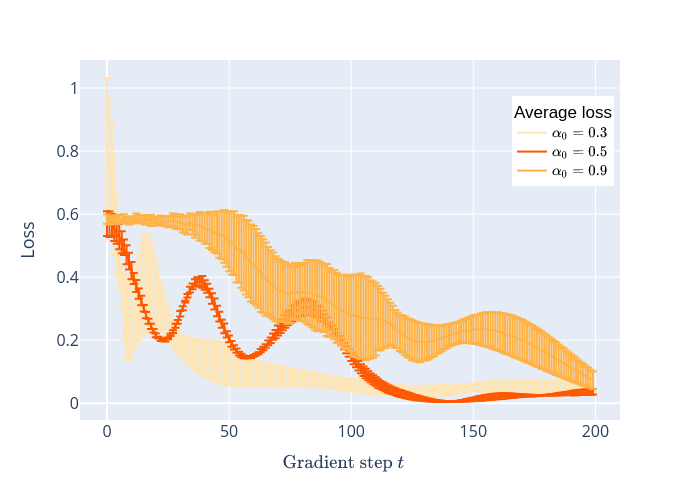}
    \caption{Left: scaling $\alpha_t$ of $\deltal$ against the initial scaling $\alpha_0$ for different training times. Right: Average loss value across depths $L\in\left\{2^k : k \in \left[3, 12\right] \right\}$ for different initializations $\alpha_0$, as a function of the number of gradient steps $t$.}
    \label{fig:trainable-delta-scaling}
\end{figure}
We observe in Figure~\ref{fig:trainable-delta-scaling} (left) that $\alpha_t$ tend to get closer to $1/2$ as $t$ increases. However, this is far from being exact, even though the networks have all converged, see Figure~\ref{fig:trainable-delta-scaling} (right). It is interesting to note that $\alpha_0 = 1/2$ is a \textit{fixed point}, meaning that the networks initialized with this scaling will keep $\alpha_t \approx 1/2$ during the entire training.
The second experiment is to let $\deltal = L^{-\alpha_0}$ at initialization and keep it fixed during training, i.e. $\alpha_t = \alpha_0$ for each $t$. We thus have weights $A^{(L)}(0)$ that scale like $L^{-\beta_0}$ initially, and that are updated with $\lr_L(t) \nabla_{A_k} J_L(A^{(L)}(t)) \propto L^{-\alpha_0} J_L(A^{(L)}(t))^{1/2}$ by Lemma \ref{lemma:gradient-upper}. Thus, it is reasonable to expect that if $\beta_0 > \alpha_0$, and the loss $J_L$ at small times $t$ is independent of the depth, then $\beta_t \approx \alpha_0$ for small times $t$.
\begin{figure}[!ht]
    \centering
    \includegraphics[width=0.48\textwidth]{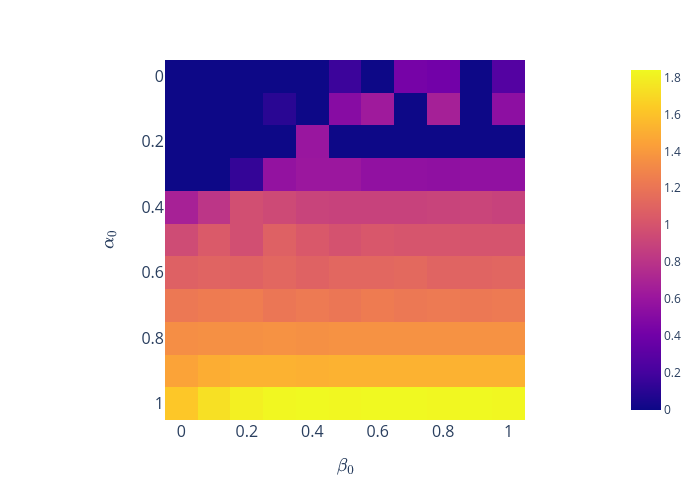}
    \includegraphics[width=0.48\textwidth]{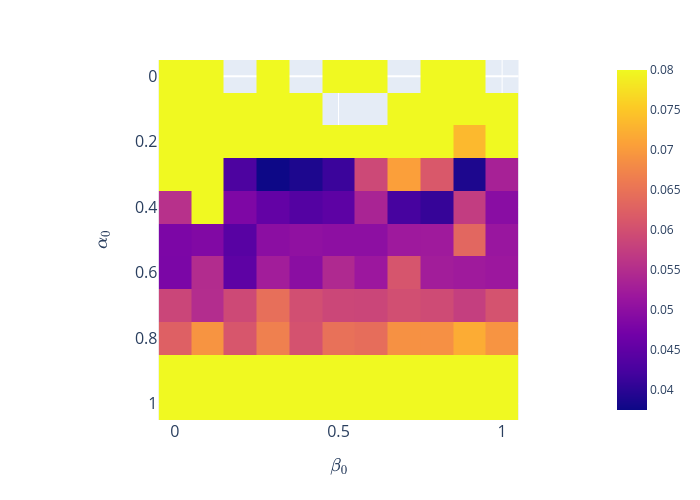}
    \caption{Both figures: horizontal axis is the initial scaling $\beta_0$ of the weights $A$, and the vertical axis is the fixed scaling $\alpha_0$ of $\deltal$. Left: Final total scaling $\alpha_0 + \beta_T$. Right: Average final loss after $T=200$ epochs. The depths at which we train our networks are $L\in\left\{2^k : k \in \left[3, 10\right] \right\}$.}
    \label{fig:fixed-delta-scaling}
\end{figure}
In fact, we observe in Figure~\ref{fig:fixed-delta-scaling} (left) that the total scaling $\alpha_0 + \beta_T$ is independent of $\beta_0$ and is roughly equal to $2\alpha_0$. We observe in Figure~\ref{fig:fixed-delta-scaling} (right) that the parameters that gives the best performance is around $\alpha_0=1/2$, again independently of $\beta_0$. This is expected, as \[
h_{k}^{(L)} - h_{k-1}^{(L)} = \deltal \sigma_d\left(A_k^{(L)} h_{k-1}^{(L)} \right) \propto L^{-\alpha_0 - \beta},
\]
so the final scaling of the increments of the hidden states is roughly $2\alpha_0$, which should be around $1$ to guarantee stability of the large depth limit.

\subsection{Rate of convergence}

We now verify that the convergence rates of gradient descent agree with the theoretical rates derived in Remark \ref{rmk:learning-rates}. To do so, we run our experiments with different initial learning rates, and take the average loss curve across the depths. We then plot the number of gradient steps needed to reach a certain loss level. 
\begin{figure}[!ht]
    \centering
    \includegraphics[width=0.48\textwidth]{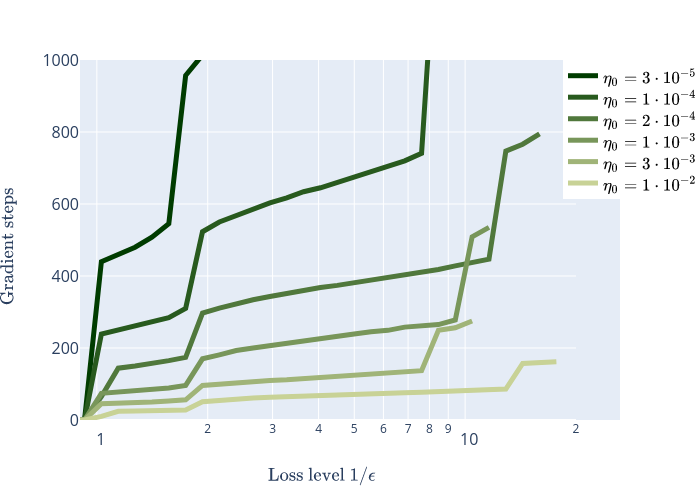}
    \includegraphics[width=0.48\textwidth]{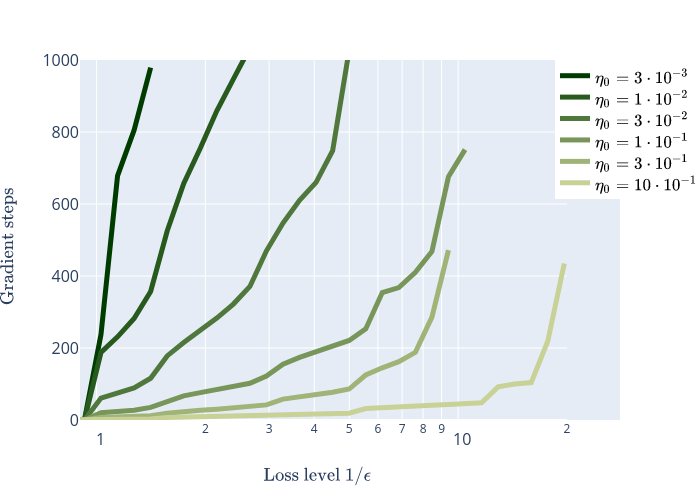}
    \caption{Both figures: horizontal axis is the inverse loss level $1/\epsilon$, in log-scale, and the vertical axis is the number of gradient steps needed for the average loss to drop below $\epsilon$. The average is taken over the depths $L\in\left\{2^k : k \in \left[3, 10\right] \right\}$. Left: constant learning rates $\lr_L(t) = \lr_0$. Right: decaying learning rates $\lr_L(t) = \lr_0 (t+1)^{-1}$.}
    \label{fig:convergence-rates}
\end{figure}
We observe in Figure \ref{fig:convergence-rates} that the number of gradient steps needed to attain a given level $\epsilon$ is linear in $\log(1/\epsilon)$ for constant learning rates, and exponential in $\log(1/\epsilon)$ for learning rates decaying like $1/t$. We also see that in both cases, the rate of convergence is inversely proportional to the initial learning rate $\eta_0$.

\subsection{Emergence of regularity of weights as a function of the layer index} \label{sec:regularity-limit}

Recall the results of Proposition \ref{prop:tanh_continuity} stating that under condition \eqref{ass:scaling-limit}, the rescaled trained weights $L^{1/2} A^{(L)}_{\floor{Ls}}(T)$ converge to a limit $\overbar{A}^*_s$ that has finite $2$-variation. We verify that condition \eqref{ass:scaling-limit} holds by running experiments for varying depths and looking at the quantities
\begin{equation*}
\overbar{f}^{(L)}(t) \coloneqq \frac{1}{2}\sum_{k=1}^L \norm{A^{(L)}_k(t)}_F^2 \quad \text{and} \quad \overbar{g}^{(L)}(t) \coloneqq \frac{1}{2}L \sum_{k=1}^{L-1} \norm{A^{(L)}_{k+1}(t) - A_k^{(L)}(t)}_F^2.
\end{equation*}

\begin{figure}[!ht]
    \centering
    \includegraphics[width=0.48\textwidth]{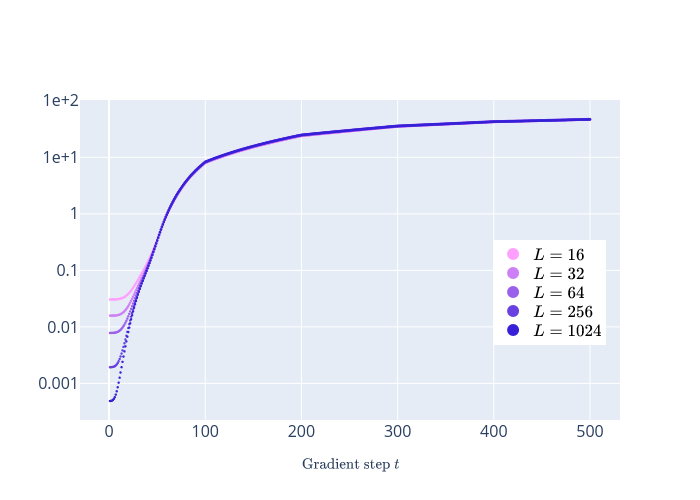}
    \includegraphics[width=0.48\textwidth]{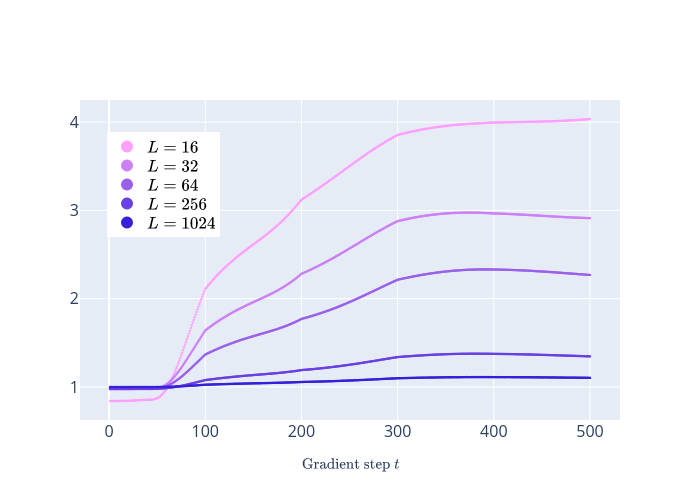}
    \caption{ Evolution of weight norms along gradient descent path for different depths $L\in\left\{2^4, 2^5, 2^6, 2^8, 2^{10} \right\}$. Left: $L^2$-type norm $\overbar{f}^{(L)}(t)$ as a function of gradient iterations. Right: Quadratic variation-type norm $\overbar{g}^{(L)}(t)$ as a function of  gradient iterations.}
    \label{fig:behaviour-norms}
\end{figure}
We observe in Figure \ref{fig:behaviour-norms} that at initialization $t=0$, the sum of the squared norms $\overbar{f}^{(L)}$ is $\OO(L^{-1})$, and becomes $\OO(1)$ during training $t\gg 1$. However, the smoothness of the weights as measured by $\overbar{g}^{(L)}(t)$ is constant with $t$ for large $L$. That means, the conservation of smoothness during training is a feature of the architecture (smooth activation function) and of gradient descent, not of the particular weight initialization nor of a particular scaling.

We observe in Figure \ref{fig:smoothing-effects} that as $L\to\infty$, the rescaled trained weights converge to a limit $\overbar{A}^*$. This is a striking result, indicative of the stability of this network architecture \cite{HR2018}: there is no \textit{a priori} reason that networks with different depths and trained independently of each other should behave similarly. The limiting behaviour of trained weights of residual networks with a smooth activation function was first observed in \cite{CCRX21}, where the limit is explicitly derived and proved.

\begin{figure}[!ht]
    \centering
    \includegraphics[width=0.7\textwidth]{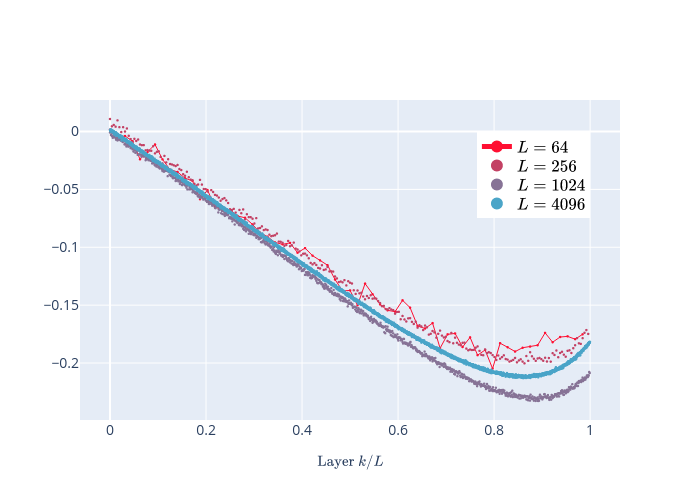}
    \caption{Scatter plot of the rescaled weights $L^{1/2} A^{(L)}_{k, (7, 18)}(T)$ for different values of $L\in\left\{4^x : x\in [3, 6] \right\}$ at the end of the training $T=500$. Horizontal axis is the scaled layer index $k/L$.}
    \label{fig:smoothing-effects}
\end{figure}
\section{Conclusion}

We prove linear convergence of gradient descent to a global minimum of the training loss for deep residual networks with constant layer width and smooth activation function. We further show that if the trained weights, as a function of the layer index, admits a scaling limit as the depth of the network tends to infinity, then it has finite $2-$variation. \\
A natural question to investigate next is the generalization capability of the trained weights obtained by gradient descent, which we characterize in this work. Indeed, it is still an open question whether the weights obtained by gradient descent admit the tightest generalization gap among all the other global minima. Also, our work can be generalized to study other residual architectures (for example with ReLU activation) by looking at alternative norms along the gradient descent path.

\bibliographystyle{siam}
\bibliography{ref}

\newpage

\appendix

\addcontentsline{toc}{section}{Appendix}
\addtocontents{toc}{\protect\setcounter{tocdepth}{-1}}

\section{Gradient of the loss function with respect to parameters} \label{app:gradient_derivation}

Let $x,y\in \R^d$ and $\parms^{(L)}\in\R^{L \times d \times d}$. We want to compute the gradient of $\ell(y, \widehat{y}(x, \parms^{(L)}))$ with respect to the network parameters $\{ \parms^{(L)}_k : k=1, \ldots, L \}$. Fix $1\leq k \leq L$ and $1 \leq m,n \leq d$. We first observe that 
\begin{align*}
    \frac{\del \ell}{\del \parms^{(L)}_{k, mn}}\left(y, \widehat{y}(x, \parms^{(L)})\right) &= \nabla_{\widehat{y}} \, \ell \left(y, \widehat{y}\left(x, \parms^{(L)}\right)\right)^{\top} \frac{\del h_L^{(L)}}{\del h^{(L)}_k} \frac{\del h_k^{(L)}}{\del \parms^{(L)}_{k, mn}}.
\end{align*}
By induction, we obtain
\begin{align}
  M_{k}^{(L)} \coloneqq \frac{\del h_L^{(L)}}{\del h_k^{(L)}} &= \prod_{j=k+1}^{L} \frac{\del h_{j}^{(L)}}{\del h_{j-1}^{(L)}} \nonumber \\
  &= \prod_{j=k+1}^{L} \left(I_d + \deltal  \frac{\del}{\del h_{j-1}^{(L)}} \sigma_d\left(\parms^{(L)}_{j} h_{j-1}^{(L)} \right) \right) \nonumber \\
  &= \prod_{j=k+1}^{L} \left( I_d + \deltal  \diag\left(\nabla\sigma_d\left(\parms^{(L)}_{j} h_{j-1}^{(L)} \right) \right) \parms^{(L)}_{j} \right). \label{app:eq:def_M}
\end{align}
We also have \[
\frac{\del h_k^{(L)}}{\del \parms^{(L)}_{k, mn}} = \deltal \sigma'\left( \left( \parms^{(L)}_k h_{k-1}^{(L)} \right)_m \right) h_{k-1, n}^{(L)}  e_m \in \R^d.
\]
Denote $\dot{\sigma}^{(L)}_{k, m} \coloneqq \sigma'\left( \left( \parms^{(L)}_k h_{k-1}^{(L)} \right)_m \right)$. Regrouping everything, we get 
\begin{equation} \label{app:eq:gradients}
\frac{\del \ell}{\del \parms^{(L)}_{k, mn}} = \deltal \, h_{k-1,n}^{(L)}  \, \dot{\sigma}^{(L)}_{k, m} \, \nabla_{\widehat{y}} \, \ell \left(y, \widehat{y}\left(x, \parms^{(L)} \right) \right)^{\top}  M_{k}^{(L)} e_m.
\end{equation}

\section{Boundedness of hidden states and Jacobians} \label{app:proof_forward _backward} This section contains two useful results for our analysis.

\begin{lemma} \label{lem:3_4_bounded_forward}
Let $\parms^{(L)}\in\R^{L\times d \times d}$ and $\Ca > 0$ such that $L\geq 5 \Ca$ and \[
\norm{\parms^{(L)}}_{F, \infty} = \max_{k=1, \ldots, L} \norm{\parms^{(L)}_k}_F \leq \Ca L^{-1/2}.
\]
Then, under Assumption \ref{3_4_assumptions_1} \textit{(i)}--\textit{(ii)}, we have that for all $x\in\R^d$ and for every $k=1, \ldots, L$,
\[
\norm{x}_2 e^{-2\Ca} \leq \norm{h^{x, \, (L)}_k}_2 \leq \norm{x}_2 e^{1.1\Ca} \hspace{1cm} \mbox{and} \hspace{1cm} \norm{M^{x, \, (L)}_{k} e_m}_2 \leq e^{\Ca}.
\]
\end{lemma}

Note that we did not try to optimize the constants in front of the bounds, and one can easily sharpen them if needed.

\begin{proof}
We follow the same lines as \cite{ALS2019}. Fix $L \geq 5 \Ca$. In the proof, we omit the explicit dependence in $L$. First, note that we can write the logarithm of the norm of the hidden state as follows:
\begin{align*}
    \log \norm{h_k} &= \log \norm{x} + \frac{1}{2} \sum_{j=1}^k \log \frac{\norm{h_j}^2}{\norm{h_{j-1}}^2} \\
    &= \log \norm{x} + \frac{1}{2} \sum_{j=1}^k \log \left( 1 + \underbrace{\frac{2\deltal}{\norm{h_{j-1}}^2} \Big\langle h_{j-1}, \sigma(\parms_j h_{j-1})  \Big\rangle + \deltal^2 \frac{\norm{\sigma(\parms_j h_{j-1})}^2}{\norm{h_{j-1}}^2}}_{\eqqcolon \, \Delta_j} \right).
\end{align*}
We can bound $\Delta_j$ further:
\begin{align}
    \Delta_j &\leq 2\deltal \norm{\parms_j}_F + \deltal^2 \norm{\parms_j}_F^2 \nonumber \\ 
    &\leq 2 \Ca L^{-1} + \Ca^2 L^{-2} \leq \frac{11}{5} \Ca L^{-1}. \label{3_4_bound_delta_l}
\end{align}
The first inequality holds by Cauchy-Schwartz and Assumption \ref{3_4_assumptions_1} \textit{(ii)}, the second by hypothesis, and the third by Assumption \ref{3_4_assumptions_1} \textit{(i)}. Thus, we conclude the proof of the upper bound by noting that $\log(1+z) \leq z$ for all $z > -1$. \\ 
For the lower bound, first observe that Cauchy-Schwartz yields \[
\Delta_j \geq -2 \deltal \norm{\parms_j}_F \geq -2 \Ca L^{-1}.
\]
From \eqref{3_4_bound_delta_l}, we also have $\abs{\Delta_j} \leq \frac{11}{25} < \frac{1}{2}$, so we can use the fact that $\log(1+z) \geq z - z^2$ for all $\abs{z} < \frac{1}{2}$ to deduce that 
\begin{align*}
    \log \norm{h_k} &\geq \log \norm{x} +  \frac{1}{2} \sum_{j=1}^k \left( \Delta_j - \Delta_j^2 \right) \\
    &\geq \log \norm{x} - \Ca - \frac{121}{25} \Ca^2 L^{-1} \geq \log \norm{x} - 2 \Ca,
\end{align*}
which concludes the proof for the lower bound on the hidden states. \\
For the upper bound on the Jacobians, we apply Lemma \ref{app:lemma_spectral_frobenius} repeatedly on $M_{k}$ to get 
\begin{align*}
    \log \norm{M_{k} e_m}_2 &\leq \log \norm{e_m}_2 + \sum_{j=k+1}^L \log \norm{I_d + \deltal  \diag\left(\nabla\sigma_d\left(\parms_{j} h_{j-1} \right) \right) \parms_{j} }_2   \\ 
    &\leq \sum_{j=k+1}^L \deltal \norm{\diag\left(\nabla\sigma_d\left(\parms_{j} h_{j-1} \right) \right) \parms_{j}}_2 \leq \sum_{j=k+1}^L \deltal \norm{\parms_j}_F \leq \Ca,
\end{align*}
where we use $\norm{\cdot}_2 \leq \norm{\cdot}_F$ and Assumption \ref{3_4_assumptions_1} \textit{(ii)} in the third inequality.
\end{proof}

We deduce directly an upper bound on the loss function $J_L$ that does not depend on $L$. 
\begin{corollary} \label{app:bound_loss_fct}
Under the same hypotheses as Lemma \ref{lem:3_4_bounded_forward}, we have \[
J_L\left(\parms^{(L)}\right) \leq 1 + e^{2.2 \Ca}.
\]
\end{corollary}
\begin{proof}
By definition of the loss function and using Lemma \ref{lem:3_4_bounded_forward}, we have 
\begin{align*}
    J_L\left(\parms^{(L)}\right) &= \frac{1}{2N} \sum_{i=1}^N \norm{ y_i - \widehat{y}_L \left(x_i, \parms^{(L)}\right)}_2^2 \\
    &\leq \frac{1}{2N} \sum_{i=1}^N 2 \norm{y_i}^2 + 2 \norm{h^{x_i, \, (L)}_L}_2^2 \leq 1 + e^{2.2\Ca}.
\end{align*}
\end{proof}

\section{Upper bounds on the gradient and Hessian of the loss function} \label{app:gradient_estimation}

\begin{lemma} \label{lemma:gradient-upper}
Let $\parms^{(L)}\in\R^{L\times d \times d}$ and $\Ca > 0$ such that $L\geq 5 \Ca$ and $\norm{\parms^{(L)}}_{F, \infty} \leq \Ca L^{-1/2}$. Then, under Assumption \ref{3_4_assumptions_1} \textit{(i)}--\textit{(ii)}, for $k=1, \ldots L$, it holds that
\begin{equation*}
\norm{\nabla_{\parms_k} J_L\left(\parms^{(L)} \right)}_F^2 \leq 2d e^{4.2\Ca} L^{-1} 
J_L\left(\parms^{(L)} \right).
\end{equation*}
\end{lemma}

\begin{proof}
Fix $L \geq 5 \Ca$. In the proof, we omit the explicit dependence in $L$. We first use Cauchy-Schwartz and \eqref{app:eq:gradients} to bound the Frobenius norm. 
\begin{align*}
    \norm{\nabla_{\parms_k} J_L(\parms)}_F^2 &= \sum_{m,n=1}^d \left( \frac{\del J_L}{\del \parms_{k, mn}} (\parms) \right)^2 \\
    &\leq \sum_{m,n=1}^d \frac{1}{N} \sum_{i=1}^N \left( \frac{\del \ell}{\del \parms_{k, mn}}\left(y_i, \widehat{y}(x_i, \parms) \right) \right)^2 \\ 
    &\leq \sum_{m,n=1}^d \frac{\deltal^2}{N} \sum_{i=1}^N \left( h^{x_i}_{k-1, n} \right)^2 \norm{ \nabla_{\widehat{y}} \, \ell\left(y_i, \widehat{y}(x_i, \parms) \right) }_2^2 \norm{M_{k}^{x_i} e_m}_2^2 \\
    &= \frac{2L^{-1}}{N} 
    \sum_{i=1}^N \norm{h^{x_i}_{k-1}}_2^2 \ell(y_i, \widehat{y}(x_i, \parms)) \norm{M_{k}^{x_i}}_F^2 \\
    &\leq 2d e^{4.2\Ca} L^{-1} J_L(\parms), 
\end{align*}
where we use the fact that $2\ell(y, \widehat{y}) = \norm{ \nabla_{\widehat{y}} \, \ell\left(y, \widehat{y} \right) }_2^2$ and Lemma \ref{lem:3_4_bounded_forward} in the last inequality. \end{proof}

Finally, we derive an upper bound on the spectral norm of the Hessian of the loss function.

\begin{lemma} \label{app:hessian}
Let $\parms^{(L)}\in\R^{L\times d \times d}$ and $\Ca > 0$ such that $L\geq 5 \Ca$ and $\norm{\parms^{(L)}}_{F, \infty} \leq \Ca L^{-1/2}$. Then, under Assumption \ref{3_4_assumptions_1} \textit{(i)}--\textit{(ii)}, we have \[
    \norm{\nabla^2_{\parms} J_L\left(\parms^{(L)}\right)}_2 \leq 5d e^{4.3 \Ca}.
\]
\end{lemma}

\begin{proof}
    Fix $L \geq 5 \Ca$. In the proof, we omit the explicit dependence in $L$. We use first-order information \eqref{app:eq:gradients} to compute the second-order derivatives. Straightforward but lengthy computations show that \[
        \nabla^2_{\parms} J_L \left(\parms^{(L)} \right) = H_{\psd} + H + \widetilde{H} + \OO\left(L^{-1/2}\right),
    \]
    where $H_{\psd}, H, \widetilde{H} \in \R^{Ld^2 \times Ld^2}$ are given by the following formulae: \[
    H_{\psd} = \frac{1}{N}\sum_{i=1}^N H^{x_i}_{\psd} \quad H = \frac{1}{N}\sum_{i=1}^N H^{x_i} \quad \widetilde{H} = \frac{1}{N}\sum_{i=1}^N \widetilde{H}^{x_i},
    \] 
    where
    \begin{align*}
        \left(H^{x_i}_{\psd}\right)_{(k, \, mn) \, (k', \, m'n')} &= \deltal^2  h^{x_i}_{k-1, n} h^{x_i}_{k'-1, n'}  \dot{\sigma}_{k, x_i, m} \dot{\sigma}_{k', x_i, m'} \left(M^{x_i}_{k} e_m \right)^{\top} \left(M^{x_i}_{k'} e_{m'} \right)  \\
        H^{x_i}_{(k, \, mn) \, (k', \, m'n')} &= \deltal h^{x_i}_{k-1, n} h^{x_i}_{k-1, n'} \ddot{\sigma}_{k, x_i, m} \left( \widehat{y}(x_i, \parms) - y_i \right)^{\top} M_{k}^{x_i} e_m \one_{m=m'} \one_{k=k'} \\
        \widetilde{H}^{x_i}_{(k, \, mn) \, (k', \, m'n')} &= \deltal^2 h^{x_i}_{k-1, n} \dot{\sigma}_{k, x_i, m} \dot{\sigma}_{k, x_i, m'} \left( \widehat{y}(x_i, \parms) - y_i \right)^{\top} M_{k, -k'}^{x_i} e_m \one_{m'=n'} \one_{k<k'}.
    \end{align*}
    Here, $M_{k, -k'}^{x_i}$ is defined as the same product of matrices as $M_{k}^{x_i}$ in \eqref{app:eq:def_M}, but without the term $j=k'$. By the same reasoning as in Lemma \ref{lem:3_4_bounded_forward}, we still have $\norm{M_{k, -k'}^{x_i} e_m}_2 \leq e^{\Ca}$. \\
    We readily see that for each $i$ there exists $Q_i$ such that $H^{x_i}_{\psd} = Q_i^{\top}Q_i$, so $H_{\psd}$ is positive semi-definite. The trace of $H_{\psd}$ is straightforward to compute.
    \begin{align*}
        \tr\left(H_{\psd}\right) = \frac{1}{N} \sum_{i=1}^N \tr\left(H^{x_i}_{\psd}\right) &= \frac{\deltal^2}{N} \sum_{i=1}^N \sum_{k,m,n} \abs{h_{k-1, n}^{x_i}}^2 \left( \dot{\sigma}_ {k,x_i, m}\right)^2   \norm{M^{x_i}_{k} e_m}_2^2. \\ 
        &= \frac{1}{N} \sum_{i=1}^N L^{-1} 
        \sum_{k=1}^L \norm{h_{k-1}^{x_i}}_2^2 \norm{M^{x_i}_{k}}_F^2.
    \end{align*}
    We deduce that by Lemma \ref{lem:3_4_bounded_forward} that $\tr\left(H_{\psd}\right) \leq d^2 e^{4.2\Ca}$. \\ 
    The upper bound on the Frobenius norm of $H$ and $\widetilde{H}$ is no harder. 
    \begin{align*}
        \norm{H^{x_i}}^2_F &\leq \sum_{k=1}^L \sum_{m=1}^d \deltal^2 \norm{h_{k-1}^{x_i}}_2^4 \ell(y_i, \widehat{y}(x_i, \parms)) \norm{M^{x_i}_{k} e_m}_2^2 \leq d e^{6.4 \Ca} \ell(y_i, \widehat{y}(x_i, \parms)), \\ 
        \norm{\widetilde{H}^{x_i}}_F^2 &\leq  \sum_{k \neq k'} \deltal^4 \norm{h_{k-1}^{x_i}}_2^2  \ell(y_i, \widehat{y}(x_i, \parms)) d^2 e^{2\Ca } \leq d^2 e^{4.2\Ca} \ell(y_i, \widehat{y}(x_i, \parms)). \\ 
    \end{align*}
    Hence, $\norm{H}_F \leq \sqrt{2d} e^{3.2\Ca} J_L(\parms)^{1/2}$ 
    and $\lVert\widetilde{H}\rVert_F \leq \sqrt{2}d e^{2.1\Ca} J_L(\parms)^{1/2}$.  
    Using Corollary \ref{app:bound_loss_fct} and wrapping both terms together, we get 
    \begin{align*}
        \norm{\nabla^2_{\parms} J\left( \parms^{(L)} \right) }_2 &\leq \norm{H_{\psd}}_2 + \norm{H}_2 + \lVert\widetilde{H}\rVert_2 + \OO(L^{-1/2}) \\
        &\leq \tr(H_{\psd})^{1/2} + \norm{H}_F + \lVert\widetilde{H}\rVert_F + \OO(L^{-1/2}) \\
        &\leq 5d e^{4.3 \Ca}. 
    \end{align*}
\end{proof}

\section{Lower bounds on loss gradients} \label{app:lower_bound_gradient}
This section contains a supporting result for the proof of Lemma \ref{lemma:lower-bound-grad}.


\begin{lemma} \label{lem:neighbour-grad}
Let $\parms^{(L)}\in\R^{L\times d \times d}$. Under Assumption \ref{3_4_assumptions_1} \textit{(i)}--\textit{(ii)}, we have, for $k=1, \ldots, L-1$, 
\begin{align*}
\frac{\del J_L}{\del \parms_{k, mn}} - \frac{\del J_L}{\del \parms_{k+1, mn}} &= \frac{\deltal}{N} \sum_{i=1}^N h_{k-1, n}^{x_i} \left( \dot{\sigma}_{k, x_i, m} - \dot{\sigma}_{k+1, x_i, m} \right) \nabla_{\widehat{y}} \, \ell\left(y_i, \widehat{y}(x_i, \parms) \right)^{\top} M_{k+1}^{x_i} e_m  \\
    &+ \frac{\deltal^2}{N} \sum_{i=1}^N \nabla_{\widehat{y}} \, \ell\left(y_i, \widehat{y}(x_i, \parms) \right)^{\top} M_{k+1}^{x_i} \xi^{x_i, \, (L)}_{k, mn},
\end{align*}
where $\xi^{x, \, (L)}_{k, mn} \in \R^d$ satisfies
\[
\norm{ \xi^{x, \, (L)}_{k, mn} }_2^2 \leq 2\left( h^x_{k-1,n} \right)^2 \norm{\parms_{k+1} - \parms_k}_F^2 + 2 \norm{\parms_{k,n}}_2^4 \norm{h_{k-1}^x}_2^4.
\]
\end{lemma}

\begin{proof}
We use the gradient computation \eqref{app:eq:gradients} and the definition \eqref{app:eq:def_M} to get 
\begin{align*}
    \frac{\del J_L}{\del \parms_{k, mn}}\left( \parms^{(L)} \right) &=  \frac{\deltal}{N} \sum_{i=1}^N h_{k-1, n}^{x_i} \dot{\sigma}_{k, x_i, m} \nabla_{\widehat{y}}\, \ell\left(y_i, \widehat{y}(x_i, \parms^{(L)}) \right)^{\top} M_{k}^{x_i} e_m, \\
    \frac{\del J_L}{\del \parms_{k+1, mn}}\left( \parms^{(L)} \right) &= \frac{\deltal}{N} \sum_{i=1}^N \left( h_{k-1, n}^{x_i} + \deltal \sigma_{k, x_i, n} \right) \dot{\sigma}_{k+1, x_i, m} \nabla_{\widehat{y}} \, \ell\left(y_i, \widehat{y}(x_i, \parms^{(L)}) \right)^{\top} M_{k+1}^{x_i} e_m.
\end{align*}
We use the identity $M_{k}^{x_i} = M_{k+1}^{x_i} \Big( I_d + \deltal \diag\left(\dot{\sigma}_{k+1, x_i}\right)\parms_{k+1} \Big) $ and we take the difference of the two equations above to get
\begin{align*}
    \frac{\del J_L}{\del \parms_{k, mn}} - \frac{\del J_L}{\del \parms_{k+1, mn}} &= \frac{\deltal}{N} \sum_{i=1}^N h_{k-1, n}^{x_i} \left( \dot{\sigma}_{k, x_i, m} - \dot{\sigma}_{k+1, x_i, m} \right) \nabla_{\widehat{y}} \, \ell\left(y_i, \widehat{y}(x_i, \parms) \right)^{\top} M_{k+1}^{x_i} e_m  \\
    &+ \frac{\deltal^2}{N} \sum_{i=1}^N \nabla_{\widehat{y}} \, \ell\left(y_i, \widehat{y}(x_i, \parms) \right)^{\top} M_{k+1}^{x_i} \xi^{x_i, \, (L)}_{k, mn},
\end{align*}
where 
\begin{align*}
\norm{\xi^{x, \, (L)}_{k, mn}}_2^2 &\leq 2\left( h^x_{k-1,n} \right)^2 \norm{\parms_{k+1} - \parms_k}_F^2 + 2 \left( \dot{\sigma}_{k+1,x,m} \right)^2 \left( \sigma\left(\parms_k h^x_{k-1}\right)_n - \left(\parms_k h^x_{k-1}\right)_n \right)^2 \\
&\leq 2\left( h^x_{k-1,n} \right)^2 \norm{\parms_{k+1} - \parms_k}_F^2 + 2 \norm{\parms_{k,n}}_2^4 \norm{h_{k-1}^x}_2^4.
\end{align*}
We use the fact that $\abs{ \sigma(z) - z } \leq z^2$ by Assumption \ref{3_4_assumptions_1} \textit{(i)}.
\end{proof}

\section{Weight norms and loss function under gradient descent} \label{app:behaviour-norms}
This section contains the proof of Lemmas \ref{lemma-f} and \ref{lemma-g}. \\

\paragraph{Proof of Lemma \ref{lemma-f}}
Fix $L\in\N^*$. In the proof, we omit the explicit dependence in $L$. 
We use the identity $\frac{1}{2}(A^2 - B^2) = B(A-B) + \frac{1}{2}(A-B)^2$ and the gradient descent update rule to first compute \begin{equation} \label{eq:diff_ssq}
    f_{k,m}(\widetilde{\parms}) - f_{k,m}(\parms) = \underbrace{-L\lr_L \sum_{n=1}^d \parms_{k, mn} \frac{\del J_L}{\del \parms_{k, mn}} \left(\parms \right)}_{\eqqcolon \, S_{1}(\parms)} + \underbrace{\frac{L\lr_L^2}{2} \sum_{n=1}^d \left( \frac{\del J_L}{\del \parms_{k, mn}} \right)^2 \left( \parms \right)}_{\eqqcolon \, S_{2}(\parms)}.
\end{equation}
Recall that the gradient of the loss $\ell$ with respect to the parameter $\parms_{k, mn}$ at sample $\left(x, y\right)$ is given by \eqref{app:eq:gradients}, so that we can compute
\begin{equation*}
    \frac{\del J_L}{\del \parms_{k, mn}}\left( \parms \right) = \frac{\deltal}{N} \sum_{i=1}^N h_{k-1, n}^{x_i}(\parms) \, \dot{\sigma}_{k, x_i, m}(\parms) \, \nabla_{\widehat{y}} \, \ell\left(y_i, \widehat{y}\left(x_i, \parms \right) \right)^{\top} M_{k}^{x_i}(\parms) \, e_m.
\end{equation*}
Recall also from \eqref{def_G} that
\begin{equation*}
G^{x,y}_{k}(\parms) \cdot e_m = \frac{\del \ell(y, \cdot)}{\del h_k}\left( \widehat{y}(x, \parms) \right) e_m = \nabla_{\widehat{y}} \, \ell\left(y, \widehat{y}(x, \parms) \right)^{\top} M_{k}^{x}(\parms) \, e_m.
\end{equation*}
We focus on the square of first order term $S_{1,m}(\parms)$ defined above. We have 
\begin{align*}
    S_{1}(\parms)^2 &= \frac{L^2 \deltal^2 \lr_L^2}{N^2} \left( \sum_{n=1}^d \parms_{k, mn} \sum_{i=1}^N h_{k-1, n}^{x_i}(\parms) \, \dot{\sigma}_{k, x_i, m}(\parms) \, G^{\, x_i, y_i}_{k}(\parms) \cdot e_m \right)^2  \\
    &\leq L^2 \deltal^2 \lr_L^2 \sum_{n=1}^d \parms_{k, mn}^2 \sum_{n=1}^d \left( \frac{1}{N} \sum_{i=1}^N h_{k-1, n}^{x_i}(\parms) \, \dot{\sigma}_{k, x_i, m}(\parms) \, G^{\, x_i, y_i}_{k}(\parms) \cdot e_m \right)^2 \\
    &\leq 2L \deltal^2 \lr_L^2 f_{k,m}(\parms) \sum_{n=1}^d \frac{1}{N} \sum_{i=1}^N \left( h_{k-1, n}^{x_i}(\parms) \, \dot{\sigma}_{k, x_i, m}(\parms) \, G^{\, x_i, y_i}_{k}(\parms) \cdot e_m \right)^2 \\
    &\leq 2 \lr_L^2 f_{k,m}(\parms) \frac{1}{N} \sum_{i=1}^N \norm{h_{k-1}^{x_i}(\parms)}_2^2 \norm{ G^{\, x_i, y_i}_{k}(\parms) }_{\infty}^2.
\end{align*}
We used twice the Cauchy-Schwarz inequality and Assumption \ref{3_4_assumptions_1} \textit{(i)-(ii)}. Define now \begin{equation*}
r_{k}(\alpha) \coloneqq \lr_L \left( \frac{1}{N} \sum_{i=1}^N \norm{h_{k-1}^{x_i}(\alpha) }_2^2 \norm{ G^{\, x_i, y_i}_{k}(\alpha) }_{\infty}^2 \right)^{1/2}.
\end{equation*}
By similar estimations, we also upper bound the second-order term: $S_2(\alpha) \leq \frac{1}{2} r_{k}(\parms)^2$. Equation \eqref{eq:diff_ssq} then yields to 
\begin{align*}
    f_{k, m}(\widetilde{\parms}) &\leq f_{k,m}(\parms) + \abs{S_1(\parms)} + S_2(\parms) \\
    &\leq f_{k,m}(\parms) + \sqrt{2} r_{k}(\parms) f_{k,m}(\parms)^{1/2} + \frac{1}{2} r_{k}(\parms)^2 = \left(f_{k,m}(\parms)^{1/2} + \frac{1}{\sqrt{2}} r_{k}(\parms)\right)^2.
\end{align*}
\qedwhite

\paragraph{Proof of Lemma \ref{lemma-g}}

Fix $L\in\N^*$. In the proof, we omit the explicit dependence in $L$. Define $g_{k,m}\left( \parms \right) \coloneqq \frac{1}{2} L^2 \norm{ \parms_{k+1, m} - \parms_{k, m}}_2^2$ so that $g_{k} = \sum_{m=1}^d g_{k,m}$. We also omit the dependence in $\parms$ when it is clear. We use the identity $\frac{1}{2}(A^2 - B^2) = B(A-B) + \frac{1}{2}(A-B)^2$ and the gradient descent update rule to first compute 
\begin{align*}
    g_{k, m}(\widetilde{\parms}) - g_{k, m}(\parms) &= \underbrace{L^2 \lr_L \sum_{n=1}^d \left( \parms_{k+1, mn}  - \parms_{k, mn}  \right) \left( \frac{\del J_L}{\del \parms_{k, mn}} - \frac{\del J_L}{\del \parms_{k+1, mn}} \right)\left(\parms  \right)}_{ \eqqcolon \, S_{1,m}(\parms)} \\
    &\quad + \underbrace{\frac{L^2\lr_L^2}{2} \sum_{n=1}^d \left( \frac{\del J_L}{\del \parms_{k, mn}} - \frac{\del J_L}{\del \parms_{k+1, mn}} \right)^2 \left( \parms  \right)}_{\eqqcolon \, S_{2,m}(\parms)}.
\end{align*}
Next, we use Lemma \ref{lem:neighbour-grad} to estimate the difference of gradients with respect to weights in neighbouring layers. We also use the fact that $L \geq 5 \Ca$ and $\norm{\parms^{(L)}}_{F, \infty} \leq \Ca L^{-1/2}$ to apply Lemma \ref{lem:3_4_bounded_forward}. Recall the definition of $G$ in \eqref{def_G}. 
\begin{align*}
\frac{\del J_L}{\del \parms_{k, mn}} - \frac{\del J_L}{\del \parms_{k+1, mn}} &= \frac{\deltal}{N} \sum_{i=1}^N h_{k-1, n}^{x_i} \left( \dot{\sigma}_{k, x_i, m} - \dot{\sigma}_{k+1, x_i, m} \right) G^{x_i, y_i}_{k+1} \cdot e_m   \\
    &+ \frac{\deltal^2}{N} \sum_{i=1}^N \nabla_{\widehat{y}} \, \ell\left(y_i, \widehat{y}(x_i, \parms) \right)^{\top} M_{k+1}^{x_i} \xi^{x_i, \, (L)}_{k, mn},
\end{align*}
where $\xi^{x, \, (L)}_{k, mn} \in \R^d$ satisfies
\begin{equation} \label{eq:xhi}
\sum_{n=1}^d  \norm{ \xi^{x, \, (L)}_{k, mn} }_2^2 \leq 4 \Ca^2 e^{2.2\Ca} L^{-1}.
\end{equation}
We focus on the first order term $S_{1,m}(\parms)$ defined above. We have 
\begin{align*}
    S_{1,m}(\parms) &= \frac{\lr_L  \deltal L^{2}}{N} \sum_{i=1}^N G^{x_i, y_i}_{k+1} \cdot e_m \left( \dot{\sigma}_{k, x_i, m} - \dot{\sigma}_{k+1, x_i, m} \right) \sum_{n=1}^d \left( \parms_{k+1, mn} - \parms_{k, mn} \right) h_{k-1, n}^{x_i} \\ 
    &+ \frac{\lr_L \deltal^2 L^2}{N} \sum_{i=1}^N \nabla_{\widehat{y}} \, \ell\left(y_i, \widehat{y}(x_i, \parms) \right)^{\top} M_{k+1}^{x_i} \sum_{n=1}^d \left( \parms_{k+1, mn} - \parms_{k, mn} \right) \xi^{x_i}_{k, mn}. 
\end{align*}
Now, as $\sigma'$ is $1-$Lipschitz, we can write
\begin{align*}
    \abs{S_{1,m}(\parms)} &\leq \frac{\lr_L  \deltal L^{2}}{N} \sum_{i=1}^N \norm{G^{x_i, y_i}_{k+1,}}_{\infty} \abs{ \left(\parms_{k+1} - \parms_{k} \right) h_{k-1}^{x_i} }_m \abs{ \parms_{k} h_{k-1} - \parms_{k+1} h_k }_m  \\
    &+ \lr_L  \deltal^2 L^2 \left[ \frac{2}{N} \sum_{i=1}^N \ell\left(y_i, \widehat{y}(x_i, \parms) \right) \norm{M_{k+1}^{x_i}}_2^2 \norm{ \parms_{k+1, m} - \parms_{k, m}}_2^2 \sum_{n=1}^d \norm{ \xi^{x_i}_{k, mn} }_2^2   \right]^{1/2}.
\end{align*}
We now use the fact that $L \geq 5 \Ca$ and $\norm{\parms^{(L)}}_{F, \infty} \leq \Ca L^{-1/2}$ to apply Lemma \ref{lem:3_4_bounded_forward} on the second term and deduce that
\begin{align*}
    \abs{S_{1,m}(\parms)} &\leq \frac{\lr_L  \deltal L^{2}}{N} \sum_{i=1}^N \norm{G^{x_i, y_i}_{k+1}}_{\infty} \abs{ \left(\parms_{k+1} - \parms_{k} \right) h_{k-1}^{x_i} }_m^2  \\
    &+ \frac{\lr_L  \deltal L^{2}}{N} \sum_{i=1}^N \norm{G^{x_i, y_i}_{k+1}}_{\infty} \abs{ \left(\parms_{k+1} - \parms_{k} \right) h_{k-1}^{x_i} }_m \abs{ \parms_{k+1} \left( h_{k}^{x_i} - h_{k-1}^{x_i} \right) }_m \\
    &+ 2 e^{\Ca} \lr_L \left( \sum_{n=1}^d \norm{ \xi^{x_i}_{k, mn} }_2^2   \right)^{1/2} g_{k,m}(\parms)^{1/2} J_L(\parms)^{1/2}.
\end{align*}
We apply Cauchy-Schwarz to the first and second term and equation \eqref{eq:xhi} to the third term to get 
\begin{align*}
    \abs{S_{1,m}(\parms)} &\leq 2 \lr_L  L^{-1/2} \frac{1}{N}  \sum_{i=1}^N \norm{ h_{k-1}^{x_i} }_2^2 \norm{G^{x_i, y_i}_{k+1}}_{\infty}  g_{k,m}(\parms)  \\
    &+ \lr_L  \deltal L^{2} \left[ \frac{1}{N} \sum_{i=1}^N \norm{G^{x_i, y_i}_{k+1}}_{\infty}^2 \norm{ \parms_{k+1, m} - \parms_{k, m} }_2^2 \norm{h_{k-1}^{x_i}}_2^2 \norm{\parms_{k+1, m}}_2^2 \deltal^2 \norm{\sigma_{k, x_i}}_2^2 \right]^{1/2} \\
    &+ 4 \Ca e^{2.1\Ca} \lr_L  L^{-1/2} g_{k,m}(\parms)^{1/2} J_L(\parms)^{1/2}.
\end{align*}
We now use Lemma \ref{lem:3_4_bounded_forward} and the identity $G^{x_i, y_i}_{k+1} = M_{k+1}^{x_i} \left( \widehat{y}(x_i, \parms) - y_i \right)$ to estimate the second term in the RHS:
\begin{align*}
    \abs{S_{1,m}(\parms)} &\leq 2 \lr_L  L^{-1/2} \frac{1}{N}  \sum_{i=1}^N \norm{ h_{k-1}^{x_i} }_2^2 \norm{G^{x_i, y_i}_{k+1}}_{\infty}  g_{k,m}(\parms) \\ 
    &+ 2 \Ca^2 e^{3.2\Ca} \lr_L L^{-1} g_{k,m}(\parms)^{1/2} J_L(\parms)^{1/2} + 4 \Ca e^{2.1\Ca} \lr_L  L^{-1/2} g_{k,m}(\parms)^{1/2} J_L(\parms)^{1/2}.
\end{align*}
Thus,
\begin{align*}
    \abs{S_{1,m}(\parms)} &\leq 2 \lr_L  L^{-1/2} \frac{1}{N}  \sum_{i=1}^N \norm{ h_{k-1}^{x_i} }_2^2 \norm{G^{x_i, y_i}_{k+1}}_{\infty}  g_{k,m}(\parms)  \\
    &+ 2 \Ca e^{2.1\Ca} \lr_L \left( \Ca e^{1.1\Ca} L^{-1} + 2 L^{-1/2} \right) g_{k,m}(\parms)^{1/2}  J_L(\parms)^{1/2}.
\end{align*}
Define 
\begin{align*}
    r_{k}(\parms) &\coloneqq \lr_L  \frac{1}{N} \sum_{i=1}^N \norm{h_{k-1}^{x_i}}_2^2 \norm{G^{x_i, y_i}_{k+1}}_{\infty} \\
    \EE_{k,m}(L, d, \parms) &\coloneqq \Ca e^{2.1\Ca} \lr_L \left( \Ca e^{1.1\Ca} L^{-1} + 2 L^{-1/2} \right) g_{k,m}(\parms)^{1/2}  J_L(\parms)^{1/2}. 
\end{align*}
We then have $\abs{S_{1,m}(\parms)} \leq 2L^{-1/2} g_{k,m}(\parms) r_k(\parms) + 2 \EE_{k,m}(L, d, \parms)$. We use similar techniques to derive the upper bound $S_{2,m}(\parms) \leq L^{-1} g_{k,m}(\parms) r_{k}(\parms)^2 + \OO( \EE_{k,m}(L, d, \parms))$. Hence, we deduce the following recurrence relation. 
\begin{equation*}
    g_{k,m}(\widetilde{\parms}) \leq  g_{k,m}(\parms) + \abs{S_{1,m}(\parms)} + S_{2,m}(\parms) \leq g_{k,m}(\parms) \left( 1 + L^{-1/2} r_{k}(\parms) \right)^2 + \OO\left(\EE_{k,m}(L, d, \parms) \right).
\end{equation*}
Summing over $m=1, \ldots, d$ and using Cauchy-Schwarz on the $\EE_{k,m}$ terms, we get 
\begin{equation*}
    g_{k}(\widetilde{\parms})  \leq g_{k}(\parms) \left( 1 + L^{-1/2} r_{k}(\parms) \right)^2 + \OO\left( \EE_{k}(L, d, \parms) \right),
\end{equation*}
where 
\begin{equation} \label{app:eq:error-g}
\EE_{k}(L, d, \parms) \coloneqq \Ca e^{2.1\Ca} \lr_L \left( \Ca e^{1.1\Ca} L^{-1} + 2 L^{-1/2} \right) g_{k}(\parms)^{1/2}  J_L(\parms)^{1/2}.
\end{equation}
\qedwhite

\section{Supporting lemma for Theorem~\ref{thm:global-convergence}}

\begin{lemma} \label{bound-loss-function}
Let $\parms^{(L)}(0)\in\R^{L\times d \times d}$ be any weight initialization. Define recursively $\parms^{(L)}(t+1) = \parms^{(L)}(t) - \lr_L(t) \nabla_{\parms} J_L \hspace{-2pt} \left( \parms^{(L)}(t) \right)$ for $t = 0, \ldots, T-1$. Assume that for all $t = 0, \ldots, T-1$, there exist $\Ca(t), \underline{c}(t), \overbar{c}(t) > 0$ such that 
\vspace{3pt}
\begin{itemize}
    \setlength\itemsep{0.3em}
    \item[(i)] $L \geq 5\max_{t<T}\Ca(t)$,
    \item[(ii)] $\norm{\parms^{(L)}(t)}_{F, \infty} \leq \Ca(t) L^{-1/2}$, and
    \item[(iii)] $\norm{\nabla_{\parms^{(L)}} J_L\hspace{-2pt}\left(\parms^{(L)}(t) \right)}_F^2 \geq \underline{c}(t) J_L\hspace{-2pt}\left(\parms^{(L)}(t)\right) - \overbar{c}(t) L^{-1}$.
\end{itemize}
\vspace{6pt}
Then, under Assumption \ref{3_4_assumptions_1} \textit{(i)}--\textit{(ii)}, if the learning rates satisfy: \[
\lr_L(t) < \min \left( \frac{1}{2} \Ca(t) e^{-3.2 \Ca(t)}, \, \frac{1}{10} \underline{c}(t) d^{-1} e^{-8.5\Ca(t)} \right), 
\] 
we have, for each $t=0, \ldots, T$:
\begin{equation} \label{eq:decrease-loss}
J_L\hspace{-2pt}\left(\parms^{(L)}(t)\right) \leq \exp\left( - \frac{1}{2} \sum_{t'=0}^{t-1} \underline{c}(t') \lr_L(t') \right) J_L\hspace{-2pt}\left( \parms^{(L)}(0) \right) + L^{-1} \sum_{t'=0}^{t-1} \overbar{c}(t') \lr_L(t').
\end{equation}
\end{lemma}

\begin{proof}
Fix $L \geq 5\max_{t<T}\Ca(t)$. We omit the explicit dependence in $L$. Fix $t\in\left[0, T\right)$. We first view $\parms(t), \, \nabla_{\parms} J_L(\parms(t)) \in \R^{L\times d \times d}$ as vectors in the Euclidean space $\R^{Ld^2}$, and we get by hypothesis and by Lemma \ref{lemma:gradient-upper} that
\begin{align*}
    \norm{\vecc(\parms(t))}_2 = \norm{\parms(t)}_F &= \left( \sum_{k=1}^L \norm{ \parms_k(t) }_F^2 \right)^{1/2} \leq \Ca(t), \\
    \underline{c}(t) J_L(\parms(t)) -  \overbar{c}(t) L^{-1}  &= \norm{\nabla_{\parms} J_L(\parms(t))}_F^2 \leq 2 e^{4.2 \Ca(t)} J_L(\parms(t)). 
\end{align*}

We want to use Lemma \ref{hessian-lemma} with $p=Ld^2$, $R=\Ca$,  
$x_0 = \parms(t)$ and $x = \parms(t) - \eta_L(t) \nabla_\parms J_L(\parms(t))$. For this, we need to check two assumptions. The first is an upper bound on the spectral norm of the Hessian of $J_L$, which we get from Lemma \ref{app:hessian}.
\begin{equation*}
    H_{\infty}(t) = \sup_{\norm{\parms'}_F \leq \, \Ca(t)} \norm{\nabla^2 J_L(\parms')}_2 \leq 5d  e^{4.3 \Ca(t)}. 
\end{equation*}
The second is an upper bound on the norm of $x-x_0 = - \lr_L(t) \nabla_\parms J_L(\parms(t))$, which we get from Lemma \ref{lemma:gradient-upper}.
\begin{align*}
    \lr_L(t) 
    \norm{\nabla_\parms J_L(\parms(t))}_2 &\leq \sqrt{2} \lr_L(t)  e^{2.1 \Ca(t) } J_L(\parms(t))^{1/2} \\
    &\leq \sqrt{2} \lr_L(t) e^{2.1 \Ca(t)} \left(1 + e^{2.2 \Ca(t)} \right)^{1/2} \\ 
    &\leq 2 \lr_L(t) e^{3.2 \Ca(t)} 
    \leq \Ca(t),
\end{align*}
where the second inequality comes from Corollary \ref{app:bound_loss_fct} and the third inequality from the fact that $(1+z)^{1/2} \leq (2z)^{1/2}$ for $z\geq 1$. Hence, we can apply Lemma \ref{hessian-lemma} and deduce that
\begin{align*}
J_L\left(\parms(t+1) \right) &= J_L \Big( \parms(t) - \lr_L(t) 
\nabla_\parms J_L(\parms(t)) \Big) - J_L(\parms(t)) \\
&\leq J_L \left( \parms(t) \right) - \lr_L(t) 
\norm{\nabla_\parms J_L(\parms(t))}_F^2 + \frac{1}{2} H_{\infty}(t) \lr_L(t)^2
\norm{\nabla_\parms J_L(\parms(t))}_2^2 \\
&\leq \left( 1 - \underline{c}(t) \lr_L(t) + 5d \lr_L(t)^2 
e^{8.5\Ca(t)} \right) J_L\left( \parms(t) \right) + \overbar{c}(t) \lr_L(t) L^{-1}.
\end{align*}
To finish the proof, we apply Lemma \ref{lemma:gronwall} \textit{(i)} with \[
u_L(t) \coloneqq \underline{c}(t) \lr_L(t) - 5d \lr_L(t)^2 e^{8.5\Ca(t)} \geq \frac{1}{2}  \underline{c}(t) \lr_L(t)  > 0, 
\]
and the fact that $1-x\leq e^{-x}$. Hence, 
\begin{align*}
J_L(\parms(T)) &\leq \exp\left( - \sum_{t=0}^{T-1} u_L(t) \right) J_L(\parms(0)) +  L^{-1} \sum_{t=0}^{T-1} \overbar{c}(t) \lr_L(t) \\
&\leq \exp\left( - \frac{1}{2} \sum_{t=0}^{T-1} \underline{c}(t) \lr_L(t) \right) J_L(\parms(0)) + L^{-1} \sum_{t=0}^{T-1} \overbar{c}(t)  \lr_L(t).
\end{align*}
\end{proof}

\section{Auxiliary results}

\begin{lemma} \label{app:lemma_spectral_frobenius}
    For any $A \in \R^{m \times n}$ and $B \in \R^{n \times p}$, we have \[
        \norm{AB}_F \leq \norm{A}_2 \norm{B}_F.
    \]
\end{lemma}
\begin{proof}
    Let $B = \left[b_1, \ldots, b_p \right]$ the columns of $B$. Then $\norm{B}_F^2 = \sum_{i=1}^p \norm{b_i}_2^2$. We use the fact that the spectral norm is compatible with the Euclidian norm to deduce \[
    \norm{AB}_F^2 = \sum_{i=1}^p \norm{Ab_i}_2^2 \leq \sum_{i=1}^p \norm{A}^2_2 \norm{b_i}^2_2 = \norm{A}^2_2 \norm{B}_F^2.
    \]
\end{proof}

\begin{lemma} \label{app:lemma_product_matrix}
Let $x\in\R^d$ and $\left\{ A_k : k=1, \ldots, L \right\} \subset \R^{d\times d}$ such that $\max_k \norm{A_k}_2 < 1$. Then \[ 
\norm{ \left[ \prod_{k=1}^L (I_d + A_k) \right] x \, }_2 \geq \norm{x}_2 \prod_{k=1}^L (1 - \norm{A_k}_2 ).
\]
\end{lemma}
\begin{proof}
First observe that for $A, B \in \R^{d\times d}$ and $x\in\R^d$, we have $\norm{ABx}_2 \geq \sigma_{\min}(A) \norm{Bx}_2$, where $\sigma_{\min}(A)$ is the smallest singular value of $A$. This is easy to see, as $\sigma_{\min}(A)^2$ is the smallest eigenvalue of $A^{\top}A$, so \[
 \norm{ABx}_2^2 = (Bx)^{\top} A^{\top}A (Bx) \geq \sigma_{\min}(A)^2 \norm{Bx}_2^2.
\]
Observe also that for all $A\in\R^{d\times d}$ with $\norm{A}_2 < 1$, we have $\sigma_{\min}(I_d + A) \geq 1 - \norm{A}_2 > 0$. Indeed, there exists $v\in\R^d$ such that $\norm{v}_2 = 1$ and $v^{\top} (I_d + A) v = \sigma_{\min}(I_d + A)^2$. Hence, \[
\sigma_{\min}(I_d + A) = \left( 1 + v^{\top} A v \right)^{1/2} \geq \left( 1 - \norm{A}_2 \right)^{1/2} \geq 1 - \norm{A}_2.
\]
Combining these two facts, we deduce that 
\begin{equation*}
    \norm{ \left[ \prod_{k=1}^L (I_d + A_k) \right] x \, }_2 \geq  \norm{x}_2 \prod_{k=1}^L \sigma_{\min}(I_d + A_k)  \geq \norm{x}_2 \prod_{k=1}^L (1 - \norm{A_k}_2 ).
\end{equation*}
\end{proof}

\begin{lemma} \label{hessian-lemma}
    Let $f \in C^2(\R^p)$ satisfying $\sup_{\norm{x}_2 < R} \norm{\nabla^2 f(x)}_2 \leq H_{\infty}$ for some $H_{\infty}, R > 0$. Then, for all $x\in\R^p$ such that $\norm{x - x_0}_2 < R$,
    \begin{equation*}
        \Big\vert f(x) - f(x_0) - \langle \nabla_x f(x), x - x_0 \rangle \Big\vert \leq \frac{H_{\infty}}{2} \norm{x - x_0}_2^2.    
    \end{equation*}
\end{lemma}
\begin{proof}
    We apply the fundamental theorem of calculus for line integrals between $x_0$ and $x$:
    \begin{equation*}
        f(x) - f(x_0) = \int_0^1 \big\langle \nabla_x f(x_0 + t(x-x_0)), x-x_0 \big\rangle \dd t.
    \end{equation*}
    Hence, by Cauchy-Schwartz inequality and by hypothesis,
    \begin{align*}
        \Big\vert f(x) - f(x_0) - \big\langle \nabla_x f(x_0), x-x_0 \big\rangle \Big\vert &\leq \int_0^1 \norm{\nabla_x f(x_0 + t(x-x_0)) - \nabla_x f(x_0)}_2 \norm{x-x_0}_2 \dd t \\
        &\leq \int_0^1 H_{\infty} \norm{t(x-x_0)}_2 \norm{x-x_0}_2 \dd t \\
        &= \frac{H_{\infty}}{2} \norm{x-x_0}_2^2.
    \end{align*}
\end{proof}

\begin{lemma}[Discrete Grönwall inequalities] \label{lemma:gronwall}
    Let $(u_n)_{n\in \N}, (v_n)_{n\in \N}, (w_n)_{n\in \N} \subset \R_{>0}$. Then
    \begin{itemize}
        \item[(i)] If $e_{n+1} \leq u_n e_n + v_n$ for each $n\geq 0$, then \[
            e_n \leq \left( \prod_{n'=0}^{n-1} u_{n'} \right) e_0 + \sum_{n'=0}^{n-1} \left( \prod_{n''=n'+1}^{n-1} u_{n''} \right) v_{n'}.
        \]
        \item[(ii)] If $g_0 > 0$ and $0 < g_{n+1} \leq u_n g_n + w_n g_n^{1/2}$, then \[
            g_{n}^{1/2} \leq \left( \prod_{n'=0}^{n-1} u_n^{1/2} \right) g_0^{1/2} + \frac{1}{2} \sum_{n'=0}^{n-1} \left( \prod_{n''=n'+1}^{n-1} u_{n''}^{1/2} \right) \frac{w_{n'}}{u_{n'}^{1/2}}.
        \]
    \end{itemize}
\end{lemma}
The first inequality is well-known, but we give proofs for both, for the sake of completeness.
\begin{proof}
To prove (i), we start by defining $\widetilde{e}_n = \left(\prod_{n'=0}^{n-1} u_{n'} \right)^{-1} e_n$. Then,
\begin{equation*}
    \widetilde{e}_{n+1} - \widetilde{e}_n =  \left(\prod_{n'=0}^{n} u_{n'} \right)^{-1} \left( e_{n+1} - u_n e_n \right) \leq  \left(\prod_{n'=0}^{n} u_{n'} \right)^{-1} v_n.
\end{equation*}
Hence, summing over $n$, we get 
\begin{align*}
    e_n = \left(\prod_{n'=0}^{n-1} u_{n'} \right) \widetilde{e}_n &\leq  \left(\prod_{n'=0}^{n-1} u_{n'} \right) \left( e_0 + \sum_{n'=0}^{n-1} \left(\prod_{n''=0}^{n'} u_{n''} \right)^{-1} v_{n'} \right) \\
    &= \left( \prod_{n'=0}^{n-1} u_{n'} \right) e_0 + \sum_{n'=0}^{n-1} \left( \prod_{n''=n'+1}^{n-1} u_{n''} \right) v_{n'}.
\end{align*}
To prove (ii), we simply complete the square: $u_n g_n + w_n g_n^{1/2} \leq u_n \left( g_n^{1/2} + \frac{w_n}{2 u_n} \right)^2$. Hence, \[
    g_{n+1}^{1/2} \leq u_n^{1/2} g_n^{1/2} + \frac{w_n}{2 u_n^{1/2}}.
\]
We can thus apply part (i) to $e_n = g_{n}^{1/2}$ to deduce the result.

\end{proof}

\end{document}